\documentclass{article}

\usepackage{microtype}
\usepackage{graphicx}
\usepackage{subcaption}
\usepackage{booktabs} % for professional tables

\usepackage{hyperref}
\usepackage{multirow}
\usepackage{makecell}

\usepackage[accepted]{icml2026}

\def\eg{\emph{e.g.}}
\def\wrt{\emph{w.r.t.~}}

\usepackage{amsmath}
\usepackage{amssymb}
\usepackage{mathtools}
\usepackage{amsthm}

\usepackage[capitalize,noabbrev]{cleveref}

\theoremstyle{plain}
\newtheorem{theorem}{Theorem}[section]

\newtheorem{corollary}[theorem]{Corollary}
\theoremstyle{definition}
\newtheorem{definition}[theorem]{Definition}
\newtheorem{assumption}[theorem]{Assumption}
\theoremstyle{remark}
\newtheorem{remark}[theorem]{Remark}

\usepackage[textsize=tiny]{todonotes}

\icmltitlerunning{Surrogate Gradient Adaptation in Binary Neural Networks}

\begin{document}

\twocolumn[
\small
  \icmltitle{SURGE: Surrogate Gradient Adaptation in \\ Binary Neural Networks}



  \icmlsetsymbol{equal}{*}
  \icmlsetsymbol{project_lead}{$\dagger$}
  \icmlsetsymbol{corr}{$\ddagger$}

  \begin{icmlauthorlist}
    \icmlauthor{Haoyu Huang}{ncee,equal}
    \icmlauthor{Boyu Liu}{iai,equal}
    \icmlauthor{Linlin Yang}{cuc,corr}
    \icmlauthor{Yanjing Li}{seie}
    \icmlauthor{Yuguang Yang}{seie}
    \icmlauthor{Xuhui Liu}{kaust}
    \icmlauthor{Canyu Chen}{ncee}
    \icmlauthor{Zhongqian Fu}{huawei,project_lead}
    \icmlauthor{Baochang Zhang}{iai}
  \end{icmlauthorlist}


  % \icmlaffiliation{\dagger}{Equal contribution. }
  % \icmlaffiliation{*}{National College for Excellent Engineers, Beihang University, Beijing, China}
  
  \icmlaffiliation{ncee}{National College for Excellent Engineers, Beihang University, Beijing, China}
  \icmlaffiliation{iai}{School of Artificial Intelligence, Beihang University, Beijing, China}
  \icmlaffiliation{cuc}{State Key Laboratory of Media Convergence and Communication, Communication University of China, Beijing, China}
  \icmlaffiliation{seie}{School of Electronic and Information Engineering, Beihang University, Beijing, China}
  \icmlaffiliation{kaust}{King Abdullah University of Science and Technology, Saudi Arabia}
  \icmlaffiliation{huawei}{Huawei Noah’s Ark Lab, China}
  
  % \icmlcorrespondingauthor{Baochang Zhang}{bczhang@buaa.edu.cn}
  \icmlcorrespondingauthor{Linlin Yang}{lyang@cuc.edu.cn}

  \icmlkeywords{Quantization-Aware Training, Binary Neural Networks, Mode Quantization, Model Compression}

  \vskip 0.3in
]



\printAffiliationsAndNotice{
  \textsuperscript{*}Equal contribution.
  \textsuperscript{$\dagger$}Project lead.
  \textsuperscript{$\ddagger$}Corresponding author.
}

\begin{abstract}
  The training of Binary Neural Networks (BNNs) is fundamentally based on gradient approximation for non-differentiable binarization operations (\eg, \texttt{sign} function). However, prevailing methods including the Straight-Through Estimator (STE) and its improved variants, rely on hand-crafted designs that suffer from gradient mismatch problem and information loss induced by fixed-range gradient clipping. To address this, we propose SURrogate GradiEnt Adaptation (SURGE), a novel learnable gradient compensation framework with theoretical grounding. SURGE mitigates gradient mismatch through auxiliary backpropagation. Specifically, we design a Dual-Path Gradient Compensator (DPGC) that constructs a parallel full-precision auxiliary branch for each binarized layer, decoupling gradient flow via output decomposition during backpropagation. DPGC enables bias-reduced gradient estimation by leveraging the full-precision branch to estimate components beyond STE's first-order approximation. To further enhance training stability, we introduce an Adaptive Gradient Scaler (AGS) based on an optimal scale factor to dynamically balance inter-branch gradient contributions via norm-based scaling. Experiments on image classification, object detection, and language understanding tasks demonstrate that SURGE performs best over state-of-the-art methods. 
\end{abstract}

\section{Introduction}

% Deep neural networks (DNNs) have achieved remarkable success across various domains, but their escalating computational complexity and memory requirements pose significant challenges for deployment in resource-limited scenarios.
% To address this challenge, numerous model compression techniques have been developed to enhance deployment efficiency, including low-bit quantization, pruning, knowledge distillation, and low-rank decomposition.
% In particular, quantization achieves compression not by altering model structures but through bit-width reduction, enabling hardware-efficient deployment while preserving the original network architecture.
Deep neural networks (DNNs) have achieved remarkable success across various domains \citep{he2016deep_resnet, vaswani2017attention}, with model parameters scaling from millions to billions in state-of-the-art architectures \citep{brown2020language, yang2024qwen2}.
% , including computer vision, natural language processing, and speech recognition.
However, their escalating computational complexity and memory requirements pose significant challenges for deployment in resource-limited scenarios.
% such as mobile platforms and embedded systems.
To address this challenge, numerous model compression techniques have been developed to enhance deployment efficiency \citep{he2023structured, hinton2015distilling, liuefficient, yu2017compressing}, each offering distinct trade-offs among compression ratio, inference speedup, and accuracy retention.
Different from structural compression methods (\eg, pruning), quantization \citep{esser2019lsq_cq, hubara2021accurate, Wang_2022_CVPR_cq, xu2023q} achieves compression through bit-width reduction without modifying the network architecture.
The reduced bit-width representation significantly decreases storage requirements while enabling computational acceleration via low-precision operations.
% This approach enables hardware-efficient deployment while preserving the original network architecture, making it highly compatible with existing deep learning frameworks and specialized accelerators.

As an extreme form of quantization, binarization \citep{courbariaux2015binaryconnect_bnn, courbariaux2016binarized_bnn, gong2019differentiable_approx, xu2021fda_bnn, xu2022recurrent_rbonn} represents weights and activations with 1-bit values, theoretically enabling $32 \times$ memory reduction and $58 \times$ computational acceleration compared to full-precision networks \citep{rastegari2016xnor_bnn}.
These efficiency advantages of binarization make it especially practical for edge computing devices with severely limited computational resources, and its effectiveness has been proven in diverse tasks, such as classification \citep{xu2021recu_bnn}, object detection \citep{xu2022ida_det}, and natural language understanding \citep{qinbibert}.

\begin{figure*}[t]
    \centering
    \includegraphics[width=1.0\linewidth]{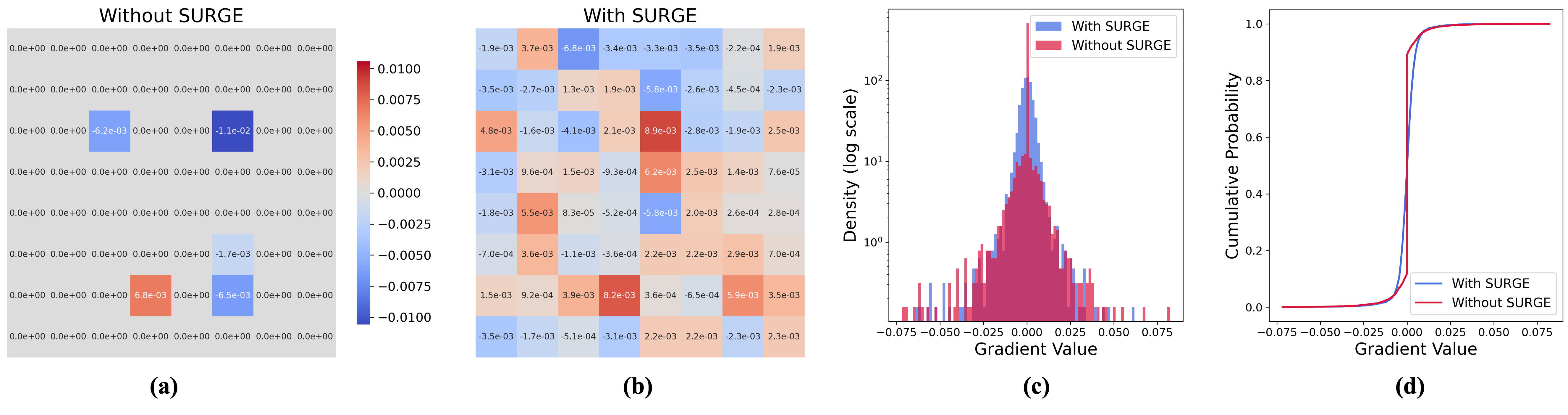}
    \caption{(a-b) Activation gradient patterns without/with SURGE (left/right); (c) Gradient distribution comparison; (d) Cumulative probability of gradients. STE provides a first-order approximation for the \texttt{sign} function's gradient and clips out-of-range activation gradients, while SURGE compensates them with a Dual-Path Gradient Compensator (a-b). SURGE also right-shifts gradient distributions of activations (c-d), validating its effectiveness in rectifying STE-induced mismatch. 
    % The shifted patterns correlate with improved parameter updates in binarized layers.
    }
    \label{fig:teaser}
\end{figure*}

Despite considerable advances, there remains a non-negligible performance gap between binary neural networks (BNNs) and their full-precision counterparts \citep{rastegari2016xnor_bnn}. This discrepancy primarily stems from the substantial representation divergence between binary and continuous-valued weights and activations.
% To address this issue, quantization-aware training (QAT) has emerged as the standard approach to optimize binary neural networks.
Specifically, the training of BNNs incorporates quantization of real-valued tensors with deterministic or stochastic binarization operations \citep{courbariaux2016binarized_bnn}. However, the non-differentiable nature and vanishing gradients of binarization operations introduce significant challenges in backpropagation.

To solve the training problem, the Straight-Through Estimator (STE) \citep{bengio2013ste_approx} provides an effective gradient approximation method for binarization operations. Specifically, STE directly substitutes the gradients of binarization operations (\eg, \texttt{sign} function) with the derivative of the \texttt{Identity} function during backpropagation, thereby enabling stable parameter optimization. Despite its prevalent application in training BNNs and low-bit networks, STE suffers from several inherent limitations that remain to be addressed. On the one hand, since the \texttt{sign} function's gradient vanishes everywhere except at zero, employing a fixed-value gradient approximation inevitably introduces estimation bias and optimization instability \citep{qin2020forward_irnet_bnn}. To reduce the gradient error of STE, subsequent approaches predominantly rely on heuristic quantizer designs \citep{liu2019circulant, gong2019differentiable_approx}, such as piecewise polynomial functions \citep{liu2018bireal_bnn} and SignSwish activation functions \citep{darabi2018bnn+}, which cannot guarantee finding the optimal gradient approximation.

% We conduct ablation studies where DPGC exclusively compensates: (1) the gradients clipped by STE, and (2) the unclipped gradients after STE's truncation. These configurations yield performance gains of +0.26\% and +0.18\% respectively compared to the baseline, demonstrating the gradient compensation benefits of DPGC.
On the other hand, during the backpropagation of STE, the gradient clipping is adopted to only preserve the gradient for inputs within the vicinity of zero (typically $[-1,1]$), which empirically improves model accuracy \citep{courbariaux2016binarized_bnn}.
% 另一方面，在采用STE训练BNNs时，一个通用的做法是使用clip函数($\text{clip}(\cdot, -1, +1)$)在0的邻域（数值小于1的地方）允许梯度传播，以获得更好的网络性能。
However, applying fixed-range gradient clipping is suboptimal for binarized representations, particularly for activation quantization, since the gradient information is discarded for values outside the clipping range \citep{qin2020forward_irnet_bnn}.
Existing binarization methods largely overlook the impact of gradient clipping range, as only a few studies propose handcrafted asymptotic functions to gradually approximate the hard binarization function \citep{gong2019differentiable_approx, qin2020forward_irnet_bnn}.
% \citep{gong2019differentiable_approx, liu2020reactnet, xu2021fda_bnn}
% handicraft asymptotic
% To solve this issue, DSQ designs a differentiable asymptotic function that gradually approximates the hard binarization function during training, dynamically narrowing the gradient propagation range across epochs.
% 然而，固定的梯度传播范围对于表征而言并不是最优的选择（DSQ），dsq是慢慢缩小区间，我们是考虑所有区间，提供必要的梯度
Consequently, merely employing STE and improved estimators \citep{rastegari2016xnor_bnn, gong2019differentiable_approx, xu2022recurrent_rbonn, jin2025parq} fails to obtain accurate gradient approximation for binarization operations, as non-negligible gradient mismatch \citep{qin2020forward_irnet_bnn} accumulates in the backward pass, necessitating explicit gradient rectification.

This paper proposes SURrogate GradiEnt Adaptation (SURGE), a novel learnable gradient compensation strategy that addresses gradient mismatch through auxiliary backpropagation. While STE or improved estimators provides surrogate gradients for binarization operations, SURGE offers enhanced gradient adaptation for binary neural networks. Specifically, we design a Dual-Path Gradient Compensator (DPGC), which constructs a parallel full-precision parameterized branch (noted as auxiliary branch) for each binarized layer (noted as main branch).
% to compensate for the gradient errors induced by STE.
In particular, DPGC decomposes each layer's output into contributions from the main branch and auxiliary branch, thus decoupling the gradient flow into two parts during backpropagation. Therefore, DPGC ensures that the auxiliary branch only affects the backward gradient while preserving the original layer outputs during the forward pass.
% fp unbiased
Compared with the binary branch, the full-precision branch can provide less biased gradients \citep{stock2021training} that compensate for STE's first-order approximation \citep{liu2023bridging} error by learning higher-order terms.
% As shown in Figure~\ref{fig:teaser}, the auxiliary branch is designed for  bias-reduced gradient estimation and eliminates fixed-range gradient clipping constraints in the binary branch.
As shown in Figure~\ref{fig:teaser}, \textbf{(a)} STE’s fixed clipping \emph{zeros vast area} of activation gradients; \textbf{(b)} with SURGE, the auxiliary branch injects compensation gradients while keeping the forward output unchanged, visibly recovering the clipped regions. Aggregated statistics in \textbf{(c)}–\textbf{(d)} show a right-shifted gradient distribution and heavier tails in the cumulative curves, indicating that SURGE restores informative gradients beyond STE’s first-order surrogate.

% permitting adaptive adjustment of clipping boundaries during activation quantization.
% 过大的辅助分支梯度可能影响
% Moreover, improperly initialized auxiliary branches may impair the convergence of the main branch.
% ablation
% We also conduct ablation studies where DPGC exclusively compensates: (1) gradients clipped by STE, and (2) gradients preserved after STE truncation. These configurations achieve performance improvements over baseline, demonstrating DPGC's compensation benefits on both gradient groups.

Moreover, large-magnitude gradients from the auxiliary path may adversely affect the convergence of the main branch.
To address this problem, we propose an Adaptive Gradient Scaler (AGS) that dynamically balances inter-branch gradient contributions via norm-based scaling, thereby ensuring stable and effective compensation. To validate the effectiveness of SURGE, we conduct comprehensive comparative experiments on two image classification benchmarks, one object detection benchmark, one suite of language understanding benchmark, and our proposed method achieves best performance over state-of-the-art. In summary, the main contributions of this work are as follows:
\begin{itemize}
    \item We propose SURrogate GradiEnt Adaptation (SURGE), a novel gradient compensation framework employing a Dual-Path Gradient Compensator to address gradient mismatch. Our method does not modify the forward-pass output and introduces no additional overhead at inference.
    \item  We introduce an Adaptive Gradient Scaler (AGS) that dynamically equilibrates gradient contributions from binary and auxiliary branches  based on theoretically derived optimal scaling factor. 
    % We also conduct a mathematical analysis to investigate the role of AGS. 
    \item Extensive experiments demonstrate that SURGE achieves state-of-the-art performance across four standard benchmarks for BNN training. Specifically, a SURGE-trained binarized ResNet-18 attains 62.0\% top-1 accuracy on ImageNet with one-stage training, surpassing previous SOTA methods by significant margins (\eg, +1.0\%, and +3.9\% top-1 accuracy improvements over ReCU and IR-Net, respectively, on ImageNet).
\end{itemize}

\section{Related Work}

\subsection{Gradient Approximation}
% 早期针对不可微操作的梯度近似采用手工设计的方法。Bengio提出的直通估计器STE通过非可微层直接传递梯度，Roth提出ProxQuant使用近端算子近似。然而这些固定近似方法无法适应输入分布而导致梯度偏差（or sub-optimal performance）
Gradient approximation serves as a cornerstone for training neural networks with non-differentiable operators, addressing challenges in discrete sampling~\citep{sutton1999policy_discrete, schulman2015gradient_discrete, athalye2018obfuscated_discrete, rezende2014stochastic_discrete}, architecture search~\citep{xie2018snas_nas, liu2018darts_nas, cai2018proxylessnas_nas}, and especially quantization~\citep{esser2019learned_lsq, gong2019differentiable_approx, liu2018bireal_bnn, liu2020reactnet, xu2022recurrent_rbonn}.
A popular family of gradient estimators is the Straight-Through Estimator (STE), which directly propagates gradients through non-differentiable functions. The idea of Straight-Through originates from the perceptron+ algorithm \citep{rosenblatt1957perceptron}, which leverages a modified chain rule and utilizes the \texttt{Identity} function as the proxy of the original derivative of a binary output function. \citep{bengio2013ste_approx} improves this method by using non-linear functions like sigmoid, and \citep{jang2016categorical_gumbel} further incorporates the Gumbel reparameterization, reparameterizes discrete variables via temperature-annealed continuous relaxation, enabling low-variance gradient estimation for categorical sampling. 
In the field of quantization, DSQ~\citep{gong2019differentiable_approx} employs parameterized sigmoid functions to progressively approximate the gradients of the non-differentiable quantization function, while LSQ~\citep{esser2019learned_lsq} introduced scaling factors for end-to-end gradient propagation, advancing low-bit quantization. BONN~\citep{zhao2022towards_BONN} integrates Bayesian optimization to guide differentiable binarization policies, and FDA-BNN~\citep{xu2021fda_bnn} converts the \texttt{sign} function into the frequency domain to mitigate the gradient mismatch.

\subsection{Binary Neural Network}

Pioneering works in binary neural networks focused either on binarization architecture design~\citep{liu2018bireal_bnn, xu2021fda_bnn, liu2020reactnet, bulat2020bats_nas_bnn, yang2020searching_slb} or training strategies~\citep{courbariaux2015binaryconnect_bnn, rastegari2016xnor_bnn, qin2020forward_irnet_bnn, xu2021recu_bnn, xu2022recurrent_rbonn}. 
In terms of architecture design, Bi-Real Net~\citep{liu2018bireal_bnn} enhances skip connections, and FDA-BNN~\citep{xu2021fda_bnn} introduces differentiable binarization units in the frequency domain. Moreover, ReActNet~\citep{liu2020reactnet} substitutes the \texttt{sign} function and PReLU~\citep{he2015delving_prelu} with RSign and RPReLU based on learnable thresholds. Approaches like BATS~\citep{bulat2020bats_nas_bnn} and SLB~\citep{yang2020searching_slb} combine BNNs with neural architecture search.
In terms of training strategies, BinaryConnect~\citep{courbariaux2015binaryconnect_bnn} and XNOR-Net~\citep{rastegari2016xnor_bnn} use the \texttt{sign} function with gradient approximation, but they cause severe information loss in forward propagation. Later, training strategies were innovated. IR-Net~\citep{qin2020forward_irnet_bnn} and ReCU~\citep{xu2021recu_bnn} use progressive quantization and feature distribution alignment, but they still face gradient mismatch in deep networks. RBONN~\citep{xu2022recurrent_rbonn} introduces a recurrent bilinear optimization for BNNs. 

Unlike prior work, our work is the first attempt to employ a Dual-Path Gradient Compensator to correct gradient mismatch in STE-based binarized networks, coupled with an Adaptive Gradient Scaler to equilibrate the gradient contribution between binary and auxiliary branches dynamically.

\begin{figure*}[t]
    \centering
    \includegraphics[width=1.0\linewidth]{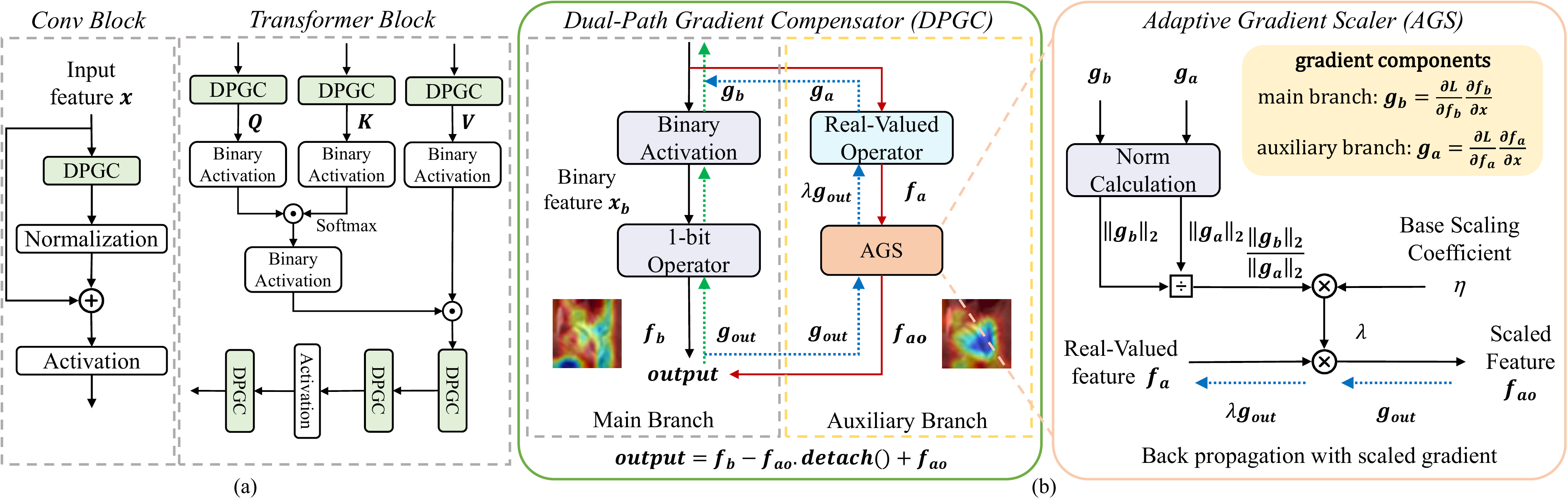}
    \caption{Overall architecture of SURGE. \textbf{(a)} Integration into common backbones (left: convolution block; right: transformer block). For visual clarity, residual connections are omitted in the Transformer block. \textbf{(b)} Component details. DPGC constructs a parallel full-precision parameterized branch (auxiliary branch, shown with red arrows for forward pass and blue arrows for backpropagation) for each binarized layer (main branch, represented by black arrows in forward pass and green arrows for backpropagation). This ensures identical output to standard BNNs while providing less biased gradients for compensation. AGS takes gradients from both branches as input (visualized through corresponding colored arrows) and dynamically balances inter-branch gradient contributions via norm-based scaling. SURGE is architecture-agnostic and applies to arbitrary binarized linear operators.}
    \label{fig:grad_comp}
\end{figure*}

\section{Preliminaries}
Consider a neural layer with weight vector $W \in \mathbb{R}^d$ and input vector $x \in \mathbb{R}^d$. The main operation in deep neural networks is expressed as:
\begin{equation}
f(x;W) = W^\top x.
\end{equation}

In binary neural networks (BNNs), we quantize $W$ and $x$ to $\{-1, +1\}^d$, thus using the efficient XNOR and Bit-count operations to replace real-valued operations. Let $ \mathbf{B}_W \in \{-1, +1\}^{d} $ and $ \mathbf{B}_{x} \in \{-1, +1\}^{d} $ denote the binarized counterparts. Network binarization aims to represent the floating-point weights and/or activations with 1 bit. In general, the quantization can be formulated as: $Q_x(x)=\alpha_x \mathbf{B}_x, Q_W(W)=\alpha_W\mathbf{B}_W$, where $\alpha_\cdot$ denotes scalars for binary values including $\alpha_w$ for weights and $\alpha_x$ for inputs. And we usually use \texttt{sign} function to binarize $W$ and $x$: $\mathbf{B}_x=\texttt{sign}(x), \mathbf{B}_W=\texttt{sign}(W)$. Following \cite{rastegari2016xnor_bnn}, the binary operation is formulated as:

\begin{equation}
f_{b}(x;\mathbf{B}_W) = Q_W(W)^\top Q_x(x) = \alpha_W \alpha_x \cdot (\mathbf{B}_W \odot \mathbf{B}_x) ,
\end{equation}

where $\odot$ denotes the inner product for vectors with bitwise operations XNOR and Bitcount.

In backpropagation, the derivative of the \texttt{sign} function is zero almost everywhere, which makes it incompatible with backpropagation, since exact gradients for the original values before binarization would be zeroed. Thus, Straight-Through Estimator (STE) \citep{bengio2013ste_approx} is generally used to train BNNs, which propagates the gradient through \texttt{Identity} function. Regarding the gradient of the loss $L$ \wrt $W$, it is approximated as 
\begin{equation}
    \frac{\partial L}{\partial W} = \frac{\partial L}{\partial \mathbf{B}_W} \cdot \frac{\partial \mathbf{B}_W}{\partial W} \approx \frac{\partial L}{\partial \mathbf{B}_W}.
\end{equation}
As for the gradient \wrt the activations, it can be formulated as 
\begin{equation}
    \frac{\partial L}{\partial x} = \frac{\partial L}{\partial \mathbf{B}_x} \cdot \frac{\partial \mathbf{B}_x}{\partial x} \approx \frac{\partial L}{\partial \mathbf{B}_x} \cdot \mathbf{1}_{\{|x| \leq 1\}},
\end{equation}
where $\mathbf{1}_{\{|x| \leq 1\}}$ is the indicator function that equals 1 when $|x| \leq 1$ and 0 otherwise. This expression corresponds to STE's first-order approximation for the sign function's gradient.

\section{Methodology}
\label{sec:method}
In this section, we describe SURGE in detail. We first introduce the Dual-Path Gradient Compensator (DPGC) module to address the gradient mismatch in STE-based training (Sec.~\ref{subsec:DPGC}), then present the Adaptive Gradient Scaler (AGS) for stable optimization (Sec.~\ref{subsec:scaling}). The complete training paradigm integrates these components while preserving standard BNN inference. 

\subsection{Dual-Path Gradient Compensator (DPGC)}
\label{subsec:DPGC}

To handle the intrinsic gradient mismatch in STE \citep{qin2020forward_irnet_bnn}, we propose a layer-wise dual-path architecture that preserves original forward computations while introducing auxiliary gradient pathways. As shown in Figure~\ref{fig:grad_comp}, DPGC constructs a parallel full-precision parameterized branch (noted as auxiliary branch) including a full-precision operator (\eg, convolution, linear, attention projection) with identical dimensions (\eg, kernel size, dimension) to the main branch, augmented with an Adaptive Gradient Scaler (AGS) module (Section~\ref{subsec:scaling}) for each binarized layer (noted as main branch). DPGC decomposes each layer's output into contributions from both the main branch (black arrow) and auxiliary branch (red arrow), thus decoupling the gradient flow into two parts during backpropagation (Eq.~\ref{eq:gradient_decouple}) (green arrow for main branch, blue arrow for auxiliary branch). Given $f_b(x) = Q_W(W_b)^\top Q_x(x)$ as the binarized computation, $f_a(x) = W_a^\top x$ as the full-precision computation, $W_a, W_b$ as the weight parameters for auxiliary branch (full-precision branch), main branch (binary branch), respectively, the combined output is:

\begin{equation}
    \textit{output} = \underbrace{f_b(x;W_b)}_{\text{Binary output}} - \underbrace{f_{ao}(x;W_a)\downarrow}_{\text{Detached compensator}} + \underbrace{f_{ao}(x;W_a)}_{\text{Active compensator}},
\end{equation}

where $f_{ao}(x) = \lambda f_a(x)$ is the scaled full-precision computation, $\lambda$ is the scale factor, $\downarrow$ is the gradient stop operator. 
The design ensures identical outputs to standard BNNs, while gradients flow through both pathways, thus providing less biased gradient estimates \citep{stock2021training} while preserving STE gradients. Upon completion of training, the auxiliary branch can be discarded, introducing no additional computational overhead during inference. Backpropagation aggregates gradients from both paths: 

\begin{equation}
\label{eq:gradient_decouple}
    \frac{\partial\mathcal{L}}{\partial x} = \underbrace{\frac{\partial\mathcal{L}}{\partial f_b}\frac{\partial f_b}{\partial x}\Big|_{\text{STE}}}_{\text{Binary gradients $g_b$}} + \lambda \underbrace{\frac{\partial\mathcal{L}}{\partial f_{a}}\frac{\partial f_{a}}{\partial x}}_{\text{Compensator gradients $g_a$}}.
\end{equation}

% Here, $g_b$ retains STE's first-order approximation \citep{liu2023bridging}, while $g_a$ from the full precision auxiliary branch captures higher-order terms and retains more gradient information excluded by clipping. 
% Here, $\left.\frac{\partial f_b}{\partial x}\right|_{\text{STE}}$ indicates the STE surrogate for the theoretically unavailable binary-branch derivative \citep{liu2023bridging}. Thus, $g_b$ is a first-order gradient approximation, while the full-precision auxiliary branch yields $g_a$ that captures higher-order terms and recovers gradients removed by clipping.
Here, $\left.\frac{\partial f_b}{\partial x}\right|_{\text{STE}}$ serves as the STE surrogate for the theoretically intractable binary-branch derivative \citep{liu2023bridging}. The term $g_b$ is thereby a first-order approximation, while the full-precision auxiliary branch yields $g_a$ to capture higher-order terms and recover gradients removed by clipping.

\subsection{Adaptive Gradient Scaler (AGS)}
\label{subsec:scaling}

The raw combination of $g_b$ and $g_a$ risks unstable training due to varying magnitude ratios between paths, and the large-magnitude gradients from the auxiliary path may adversely affect the convergence of the main branch. We address this through a novel mechanism that dynamically balances inter-branch gradient contributions with norm-based adaptive scaling factor $\lambda_{\mathrm{AGS}}$, thereby ensuring stable and effective compensation: 

% \begin{equation}
%     \tilde{g} = g_b + \lambda g_a, \quad \lambda^* = \eta\frac{\|g_b\|_2}{\|g_a\|_2 + \epsilon},
% \end{equation}

\begin{equation}
\frac{\partial\mathcal{L}}{\partial x} = g_b + \lambda_{\mathrm{AGS}} \cdot g_a,\quad 
\lambda_{\mathrm{AGS}} := \eta\frac{\|g_b\|_2}{\|g_a\|_2 + \epsilon},
\label{eq:ags_rule}
\end{equation}

where $\eta$ is the base scaling coefficient, $\epsilon = 10^{-8}$ is the numerical stabilizer, and $\lambda_{\mathrm{AGS}}$ is a practical plug-in approximation of the theoretical optimum (Theorem~\ref{thm:optimal-lambda}, Corollary~\ref{cor:norm-ratio}). This dynamic scaling preserves the directional consistency of the primary binary gradient $g_b$ while allowing auxiliary gradients $g_a$ to provide magnitude-aware compensation. 
% Specifically, the unit vector alignment $\frac{g_a}{\|g_a\|}$ remains unaffected, ensuring auxiliary updates only fine-tune the optimization trajectory rather than overriding the main gradient direction. 
Such design guarantees that the STE-based gradients dominate the parameter update process, while the auxiliary path serves as an adaptive compensator that injects higher-order gradient information without destabilizing the primary learning dynamics. In practice, the scale factor derived from gradient computation in the current iteration is used in the subsequent AGS step for adaptive parameter adjustment to optimize computational efficiency. The complete training procedure is summarized in Appendix~\ref{appendix:training_procedure}. 

% The scaling factor dynamically amplifies gradients of auxiliary branch when $\|g_a\| \ll \|g_b\|$, and suppresses interference of auxiliary branch when $\|g_a\| \gg \|g_b\|$.

\section{Theoretical Analysis}
\label{sec:theory}
This section formally establishes the theoretical foundation of gradient compensation in dual-path architectures. We begin by formulating the gradient propagation mechanism under our proposed compensation framework, followed by introducing moment-based notation for gradient statistics (Definition~\ref{def:grad_stats}) and a moment model that captures the bias/noise structure of the two gradient components (Assumption~\ref{ass:moment_model}). We then derive the theoretically optimal scaling factor for gradient compensation (Theorem~\ref{thm:optimal-lambda}) and obtain a practical norm-ratio approximation that directly motivates the AGS update rule (Corollary~\ref{cor:norm-ratio}).
% followed by rigorous definitions of gradient components to characterize their statistical properties (Definition~\ref{def:grad_stats}). Finally, we derive the theoretically optimal scaling factor for gradient compensation and demonstrate its alignment with our adaptive scaling strategy (Theorem~\ref{thm:optimal-lambda}).

Let $\mathcal{X}$ denote the input space and $W = (W_b, W_a) \in \mathbb{R}^{2d}$ represent the binarized and full-precision weights. The forward propagation becomes:

\begin{equation}
\begin{aligned}
    &f(x;W) = \\&\underbrace{Q_W(W_b)^\top Q_x(x)}_{\text{Binary path}} + \lambda\left(\underbrace{W_a^\top x}_{\text{Compensator path}} - \underbrace{W_a^\top x \downarrow}_{\text{Detached path}}\right),
\end{aligned}
\end{equation}

where $\lambda$ follows the adaptive scaling in Section~\ref{subsec:scaling}. Let $Approx$ denote a kind of STE-based gradient approximation (\eg, STE), the composite gradient combines:

\begin{equation}
    \frac{\partial\mathcal{L}}{\partial x} = \underbrace{\frac{\partial\mathcal{L}}{\partial f_b}\frac{\partial f_b}{\partial x}\Big|_{Approx}}_{g_b} + \underbrace{\frac{\partial\mathcal{L}}{\partial f_{ao}}\frac{\partial f_{ao}}{\partial x}}_{\lambda g_a}.
\end{equation}

% \begin{definition}[Gradient Components]
% \label{def:gradient_components}
% Assuming there exists $g^*$ being a better gradient, the empirical gradients satisfy:
% \begin{align}
%     &\mathbb{E}[g_b] = g^* - \delta_b,\quad \|\delta_b\|_2 \leq C\sqrt{d},   \\
%     % \mathbb{E}[g_c] &= g^* \quad (\text{Unbiased estimate}) \\
%     &\mathrm{Var}(g_b) = \sigma_b^2 I_d,\quad \mathrm{Var}(g_a) = \sigma_a^2 I_d,
% \end{align}
% where $C$ is a constant, which depends on the activation distribution. $\mathbb{E}[\cdot]$ and $\mathrm{Var}(\cdot)$ denote expectation and variance. $\delta_b$ represents the $Approx$-induced error vector, $\sigma_b$ and $\sigma_a$ are the standard deviations of gradient noise in binarized and compensator paths, respectively. $I_d$ is the $d$-dimensional identity matrix. The full list of assumptions underlying this definition is provided in Appendix~\ref{assumptions}.
% \end{definition}

\begin{definition}[Notation for gradient statistics]
\label{def:grad_stats}
Let $\mu_b := \mathbb{E}[g_b]$, $\mu_a := \mathbb{E}[g_a]$.
Define the approximation-induced bias vector of the baseline surrogate as
$\delta_b := g^* - \mu_b$.
\end{definition}

\begin{assumption}[Moment model for gradient components]
\label{ass:moment_model}
Let $\mu_b:=\mathbb{E}[g_b]$ and $\mu_a:=\mathbb{E}[g_a]$.
There exists an ideal (unobserved) reference gradient $g^*$ such that
\[
\mu_b = g^* - \delta_b,\qquad \|\delta_b\|_2 \le C\sqrt{d},
\]
and $g_b,g_a$ have finite second moments.
\end{assumption}
\noindent The definition of $g^*$ and the noise structure assumptions are detailed in Appendix~\ref{assumptions}.

% \begin{theorem}[Optimal Scaling Factor]
% \label{thm:optimal}
% Given numerical stabilizer $\epsilon = 10^{-8}$, the optimal scaling factor $\lambda^*$, minimizing total error $ \mathbb{E}[\|\tilde{g} - g^*\|^2]$, can be approximated by multiplying a small constant $\eta$ with 
% the fraction $\frac{\|g_b\|}{\|g_a\|}$ as below:
% \begin{equation}
% \label{eq:optimal}
%     \lambda^* = \frac{\langle \delta_b, \mathbb{E}[g_a] \rangle} {\|\mathbb{E}[g_a]\|^2 + d\sigma_a^2} \approx \eta \frac{\|g_b\|}{\|g_a\| + \epsilon}.
% \end{equation}
% The approximation replaces expectations with empirical mini-batch estimates.
% \end{theorem}

\begin{theorem}[Optimal scaling factor]
\label{thm:optimal-lambda}
Let $\tilde g(\lambda):= g_b+\lambda g_a$.
Under Assumption~\ref{ass:moment_model}, assume $\mu_a:=\mathbb{E}[g_a]$ exists and
$\mathrm{Var}(g_a)$ has finite trace.
Assume additionally that the mini-batch noises in $g_b$ and $g_a$ are uncorrelated in the dot-product sense:
\[
\mathbb{E}\big[(g_b-\mathbb{E}[g_b])^\top (g_a-\mathbb{E}[g_a])\big]=0 .
\]
Then any minimizer of $\mathbb{E}\|\tilde g(\lambda)-g^*\|_2^2$ is
\[
\lambda^* \;=\;
\frac{\langle \delta_b,\ \mu_a\rangle}{\|\mu_a\|_2^2+\mathrm{tr}(\mathrm{Var}(g_a))}.
\]
If $\|\mu_a\|_2^2+\mathrm{tr}(\mathrm{Var}(g_a))>0$, this minimizer is unique.
In particular, under the isotropic noise model $\mathrm{Var}(g_a)=\sigma_a^2 I_d$,
\[
\lambda^*=\frac{\langle \delta_b,\ \mu_a\rangle}{\|\mu_a\|_2^2+d\sigma_a^2}.
\]
\end{theorem}

% \begin{theorem}[Optimal scaling factor]
% \label{thm:optimal-lambda}
% Let $\tilde g(\lambda):= g_b+\lambda g_a$.
% Under Assumption~\ref{ass:moment_model} and the isotropic noise model in Appendix~\ref{assumptions}
% \[
% \mathbb{E}[g_b]=g^*-\delta_b,\quad \mathrm{Var}(g_a)=\sigma_a^2 I_d,\quad \mathrm{Var}(g_b)=\sigma_b^2 I_d
% \]
% assume additionally that the mini-batch noises in $g_b$ and $g_a$ are uncorrelated in the dot-product sense:
% \[
% \mathbb{E}\big[(g_b-\mathbb{E}[g_b])^\top (g_a-\mathbb{E}[g_a])\big]=0 .
% \]
% Then the unique minimizer of $\mathbb{E}\|\tilde g(\lambda)-g^*\|_2^2$ is
% \begin{equation}
% \begin{aligned}
% \lambda^* \;=\;&
% \frac{\langle \delta_b,\ \mu_a\rangle}{\|\mu_a\|_2^2+\mathrm{tr}(\mathrm{Var}(g_a))}
% \;=\;
% \frac{\langle \delta_b,\ \mu_a\rangle}{\|\mu_a\|_2^2+d\sigma_a^2},
% \\ \mu_a:&=\mathbb{E}[g_a].
% \end{aligned}
% \end{equation}
% \end{theorem}

\begin{corollary}[Practical norm-ratio approximation]
\label{cor:norm-ratio}
In addition, suppose during the main training phase (after a short transient):
(i) the alignment
$\cos\theta:=\frac{\langle \delta_b,\mu_a\rangle}{\|\delta_b\|_2\|\mu_a\|_2}\approx c_\theta$
is approximately stable,
(ii) the \emph{relative bias ratio}
$\beta:=\frac{\|\delta_b\|_2}{\|\mu_b\|_2}$ with $\mu_b:=\mathbb{E}[g_b]$
is bounded and slowly varying, so $\beta\approx \kappa$,
and (iii) the noise ratio $\rho:=\frac{d\sigma_a^2}{\|\mu_a\|_2^2}$ is approximately stable.
Then
\[
\lambda^* \approx \eta\frac{\|\mu_b\|_2}{\|\mu_a\|_2},
\qquad
\eta:=\frac{\kappa c_\theta}{1+\rho}.
\]
Replacing population quantities by mini-batch estimates and adding a numerical stabilizer $\epsilon>0$ yields the AGS rule
\[
\lambda_{\mathrm{AGS}} \;:=\; \eta\frac{\|g_b\|_2}{\|g_a\|_2+\epsilon}.
\]
\end{corollary}

% norm-aware used in AGS
The proof is detailed in Appendix~\ref{proof:optimal}. We now derive the practical expression $\lambda_{\mathrm{AGS}}$ for the optimal scaling factor $\lambda^*$, which is norm-based and adaptive, adopted in our AGS module (Sec.~\ref{subsec:scaling}). This analysis establishes a principled gradient scaling factor that improves the resulting gradient update and alleviates the gradient mismatch. 
% 通过理论推导，获得比较优的梯度系数，最后得到一个更优的梯度。

\begin{table}[t]
% \small
\centering
\caption{Performance comparison with the state-of-the-arts on CIFAR-10. W/A denotes the bit length of the weights and activations. }
\label{tab:cifar10}
\begin{tabular}{llcc}
\hline
\textbf{Network} & \textbf{Method} & \textbf{W/A} & \textbf{Top-1} \\ 
\hline
\multirow{6}{*}{ResNet-18} 
& Real-Valued & 32/32 & 94.8\% \\
\cline{2-4}
& RAD & 1/1 & 90.5\% \\
& IR-Net & 1/1 & 91.5\% \\
& RBNN & 1/1 & 92.2\% \\
& ReCU & 1/1 & 92.8\% \\
& \textbf{SURGE (Ours)} & 1/1 & \textbf{93.1\%} \\
\hline

\multirow{7}{*}{ResNet-20} 
& Real-Valued & 32/32 & 92.1\% \\
\cline{2-4}
& DoReFa & 1/1 & 79.3\% \\
& DSQ & 1/1 & 84.1\% \\
& SLB & 1/1 & 85.5\% \\
& IR-Net & 1/1 & 86.5\% \\
& ReCU & 1/1 & 87.4\% \\
& \textbf{SURGE (Ours)} & 1/1 & \textbf{88.0\%} \\
\hline

\multirow{9}{*}{VGG-Small} 
& Real-Valued & 32/32 & 94.1\% \\
\cline{2-4}
& XNOR-Net & 1/1 & 89.8\% \\
& DoReFa & 1/1 & 90.2\% \\
& IR-Net & 1/1 & 90.4\% \\
& RBNN & 1/1 & 91.3\% \\
& DSQ & 1/1 & 91.7\% \\
& SLB & 1/1 & 92.0\% \\
& ReCU & 1/1 & 92.2\% \\
& \textbf{SURGE (Ours)} & 1/1 & \textbf{92.5\%} \\
\hline
\end{tabular}
\end{table}

\section{Experiments}
\label{sec:experiments}
\subsection{Datasets and Implementation Details}
\label{sec:implement}
\textbf{Datasets. }
We evaluate on two standard image classification benchmarks, one object detection benchmark, and one suite of language understanding tasks to demonstrate the effectiveness: CIFAR-10 \citep{krizhevsky2009learning_cifar}, ImageNet-1K \citep{russakovsky2015imagenet}, PASCAL VOC \citep{everingham2010pascalvoc}, and GLUE \citep{wang2018glue}. More details of datasets, data augmentation, and evaluating metrics are provided in Appendix~\ref{appendix:dataset}. 

\textbf{Implementation Details. }On CIFAR-10, we evaluate our method with ResNet-18/20 \citep{he2016deep_resnet} and VGG-Small \citep{simonyan2014very_vgg}. On PASCAL VOC, we binarize Faster-RCNN \citep{ren2016faster} with a ResNet-18 backbone (with minor structural modifications shared by FP/BNN). On GLUE, we evaluate our method with BERT-base \citep{devlin2019bert}. More training details are provided in Appendix~\ref{appendix:implementation_details}. 

\subsection{Image Classification}
\label{sec:classification}
\textbf{CIFAR-10. }We first show the experimental results on CIFAR-10 with ResNet-18, ResNet-20, VGG-Small backbone in Table~\ref{tab:cifar10}. Specifically, we compare SURGE with state-of-the-art methods include RAD~\citep{ding2019regularizing_rad}, IR-Net~\citep{qin2020forward_irnet_bnn}, RBNN~\citep{lin2020rotated_rbnn}, ReCU~\citep{xu2021recu_bnn}, DoReFa~\citep{zhou2016dorefa}, DSQ~\citep{gong2019differentiable_approx}, SLB~\citep{yang2020searching_slb}, IR-Net~\citep{qin2020forward_irnet_bnn}, and XNOR-Net \citep{rastegari2016xnor_bnn}. 
We can see that SURGE outperforms all other methods in all backbones. Compared to recent ReCU, SURGE obtains a 0.3\% performance increase with ResNet-18, a 0.6\% performance increase with ResNet-20, and a 0.3\% performance increase with VGG-Small.

\begin{table*}[!t]
% \small
\centering
\caption{A performance comparison with SOTAs on ImageNet with one-stage training. W/A denotes the bit length of weights and activations. We report the Top-1 (\%) and Top-5 (\%) accuracy performances. }
\label{tab:imagenet_single_stage}
\begin{tabular}{l|c|c|c|c|c}
\hline
Network & Method & W/A & OPs ($\times 10^8$) & Top-1 & Top-5 \\
\hline
\multirow{12}{*}{ResNet-18} 
& Real-valued & 32/32 & 18.19 & 69.6 & 89.2 \\
\cline{2 - 6}
& DoReFa & 1/4 & 2.44 & 59.2 & 81.5 \\
\cline{2 - 6}
& TBN & 1/2 & 1.81 & 55.6 & 79.0 \\
\cline{2 - 6}
& BNN & \multirow{9}{*}{1/1} & \multirow{9}{*}{1.63} & 42.2 & 67.1 \\
& XNOR-Net  & & & 51.2 & 73.2 \\
& Bi-Real Net  & & & 56.4 & 79.5 \\
& IR-Net  & & & 58.1 & 80.0 \\
& BONN  & & & 59.3 & 81.6 \\
& RBNN  & & & 59.6 & 81.6 \\
& ReCU  & & & 61.0 & 82.6 \\
& RBONN  & & & 61.4 & 83.5 \\
& \textbf{SURGE (Ours)} & & & \textbf{62.0} & \textbf{83.7} \\
\hline
\end{tabular}
\end{table*}

\begin{table*}[!t]
% \small
\centering
\caption{A performance comparison with SOTAs on ImageNet with two-stage training. W/A denotes the bit length of weights and activations. We report the Top-1 (\%) and Top-5 (\%) accuracy performances. * denotes the result is from the official checkpoint.}
\label{tab:imagenet_two_stage}
\begin{tabular}{l|c|c|c|c|c}
\hline
Network & Method & W/A & OPs ($\times 10^8$) & Top-1 & Top-5 \\
\hline
\multirow{5}{*}{ResNet-18} 
& Real-valued & 32/32 & 18.19 & 69.6 & 89.2 \\
\cline{2 - 6}
& ReActNet  & \multirow{4}{*}{1/1} & \multirow{4}{*}{1.63} & 65.9 & - \\
& ReCU  & & & 66.4 & 86.5 \\
& RBONN*  & & & 66.5 & \textbf{86.7} \\
& \textbf{SURGE (Ours)} & & & \textbf{66.7} & \textbf{86.7} \\
\hline
\end{tabular}
\end{table*}

\textbf{One-Stage Training on ImageNet. }Table~\ref{tab:imagenet_single_stage} displays the performance comparison in binarizing ResNet-18 with one-stage training on ImageNet. We compare SURGE with DoReFa~\citep{zhou2016dorefa}, TBN~\citep{wan2018tbn}, BNN~\citep{courbariaux2016binarized_bnn}, XNOR-Net \citep{rastegari2016xnor_bnn}, Bi-Real Net~\citep{liu2018bireal_bnn}, IR-Net~\citep{qin2020forward_irnet_bnn}, BONN~\citep{zhao2022towards_BONN}, RBNN~\citep{lin2020rotated_rbnn}, RBONN~\citep{xu2022recurrent_rbonn}. We can see that SURGE is leading in both the top-1 and top-5 accuracies. Specifically, SURGE outperforms RBONN by 0.6\% in top-1 accuracy, achieving the best performance. 

\textbf{Two-Stage Training on ImageNet. }Table~\ref{tab:imagenet_two_stage} displays the performance comparison in binarizing ResNet-18 with two-stage training on ImageNet. We compare SURGE with ReActNet~\citep{liu2020reactnet}, ReCU~\citep{xu2021recu_bnn}, and RBONN~\citep{xu2022recurrent_rbonn}. Results show that SURGE outperforms all other methods in top-1 accuracy. Specifically, SURGE obtains a 0.2\% performance increase in top-1 over RBONN. Our SURGE demonstrates superior overall performance compared to all existing approaches. 

\begin{table*}[!h]
% \small
\centering
\caption{Performance comparison of different methods in Faster-RCNN framework with input resolution set to 1000$\times$600. $^\dagger$ denotes that the result is from our re-implementation.}
\label{tab:faster_rcnn_perf}
\begin{tabular}{l|c|c|c|c|c|c}
\hline
Framework & Backbone & Method & W/A & \makecell{Memory Usage \\ (MB)} & \makecell{OPs \\ ($\times 10^9$)} & mAP \\
\hline
\multirow{6}{*}{Faster-RCNN} 
& \multirow{6}{*}{ResNet-18} 
& Real-valued & 32/32 & 112.88 & 96.40 & 78.8 \\
\cline{3-7}
& & DoReFa-Net  & 4/4 & 21.59 & 27.15 & 73.3 \\
\cline{3-7}
& & ReActNet  & \multirow{4}{*}{1/1} & \multirow{4}{*}{16.61}  & \multirow{4}{*}{18.49} & 69.6 \\
& & LWS-Det  &  &  &  & 73.2 \\
& & IDa-Det$^\dagger$  &  &  &  & 76.5 \\
& & \textbf{SURGE (Ours)} &  &  &  & \textbf{77.0} \\
\hline
\end{tabular}
\end{table*}

\subsection{Object Detection}
On the PASCAL VOC dataset, we compare the proposed SURGE against existing state-of-the-art binarized detection methods, such as ReActNet~\citep{liu2020reactnet}, LWS-Det~\citep{xu2021layer_lws}, and IDa-Det~\citep{xu2022ida_det}, on the Faster-RCNN framework for object detection. The detection result of multi-bit quantized networks DoReFa-Net~\citep{zhou2016dorefa} is also reported. As shown in Table~\ref{tab:faster_rcnn_perf}, compared with the prior state-of-the-art IDa-Det, our method gains 0.5\% performance increase, with the same FLOPs and memory usage. Compared with the raw real-valued detectors, SURGE surpasses raw real-valued Faster-RCNN with ResNet-18 backbone (77.0\% $v.s.$ 76.4\%) by apparent computation acceleration and storage savings by $5.21\times$ and $6.80\times$.

\subsection{Language Understanding}
\label{sec:language_understanding}
On the GLUE dataset, we compare SURGE against existing state-of-the-art methods, such as BinaryBERT~\citep{bai2020binarybert}, BiBERT~\citep{qinbibert}, and BiT~\citep{liu2022bit}, on BERT. 
We can see that SURGE outperforms all other methods. Specifically, SURGE obtains a 1.4\% performance increase compared to BiT, and outperforms BiBERT by 8.9\%, achieving the best performance. 

\begin{table*}[h]
\centering
\small
\caption{Performance comparison of BERT quantization on the GLUE dev set. FP is short for full precision. $^\dagger$ denotes our re-implementation without multi-distillation techniques for fair comparison.}
\label{tab:bert}
\begin{tabular}{lccccccccccc}
\hline
Quant  & Size (MB) & FLOPs (G) & MNLI$_{m/mm}$ & QQP & QNLI & SST-2 & CoLA & STS-B & MRPC & RTE & Avg. \\
\hline
BERT (FP)          & 418  & 22.5 & 84.9/85.5 & 91.4 & 92.1 & 93.2 & 59.7 & 90.1 & 86.3 & 72.2 & 83.9 \\
\hline
BinaryBERT    & 16.5 & 0.4  & 35.6/35.3 & 66.2 & 51.5 & 53.2 & 0 & 6.1 & 68.3 & 52.7 & 41.0 \\
BiBERT        & 13.4 & 0.4  & 66.1/67.5 & 84.8 & 72.6 & 88.7 & 25.4  & 33.6  & 72.5 & 57.4 & 63.2 \\
BiT$^\dagger$    & 13.4 & 0.4  & 77.0/77.5 & 85.4 & 85.5 & 87.8 & 23.6 & 68.0 & 79.4 & 58.1 & 70.6 \\
\textbf{SURGE(Ours)}  & 13.4 & 0.4  & 77.3/77.5 & 87.1 & 86.2 & 88.6 & 24.1 & 71.7 & 80.6 & 60.6 & \textbf{72.0} \\
\hline
\end{tabular}
\end{table*}

\subsection{Ablation Study}
\textbf{Ablation on Components.} We ablate each component on CIFAR-10 using ResNet20. As shown in Table~\ref{tab:ablation_left}, the baseline achieves 87.4\% accuracy. Introducing the \textbf{Dual-Path Gradient Compensator (DPGC)} alone improves performance by +0.4\%, validating its capability to balance gradient conflicts. Subsequent integration of the \textbf{Adaptive Gradient Scaler (AGS)} adds another +0.2\%, demonstrating that AGS effectively modulates gradient magnitudes without disrupting DPGC's compensation. The hierarchical gains confirm that both mechanisms address distinct aspects of gradient optimization.

\begin{table*}[!h]
\small
\centering
\caption{Ablation Study on CIFAR-10 with ResNet20.}
\label{tab:ablation}
\hspace*{0.3cm}
\begin{subtable}{0.45\textwidth}
  \centering
  \caption{Ablation on components}
  \begin{tabular}{lc}
    \toprule
    \textbf{Method} & \textbf{Accuracy (\%)} \\
    \midrule
    Baseline & 87.4 \\
    + DPGC & 87.8 \\
    + DPGC + AGS & \textbf{88.0} \\
    \bottomrule
  \end{tabular}
  \label{tab:ablation_left}
\end{subtable}
\hfill
\begin{subtable}{0.52\textwidth}
  \centering
  \caption{Ablation on gradient compensation scope}
  \begin{tabular}{lcc}
    \toprule
    \textbf{Method} &\textbf{Scope}& \textbf{Accuracy (\%)} \\
    \midrule
    Baseline & / & 87.4 \\
    \hline
    \multirow{3}{*}{SURGE} & clipped gradients & 87.7 \\
     & unclipped gradients & 87.6 \\
     & all gradients & \textbf{88.0} \\
    \bottomrule
  \end{tabular}
  \label{tab:ablation_right}
\end{subtable}
\end{table*}

\begin{figure*}[!h]
    \centering
    \includegraphics[width=.6\linewidth]{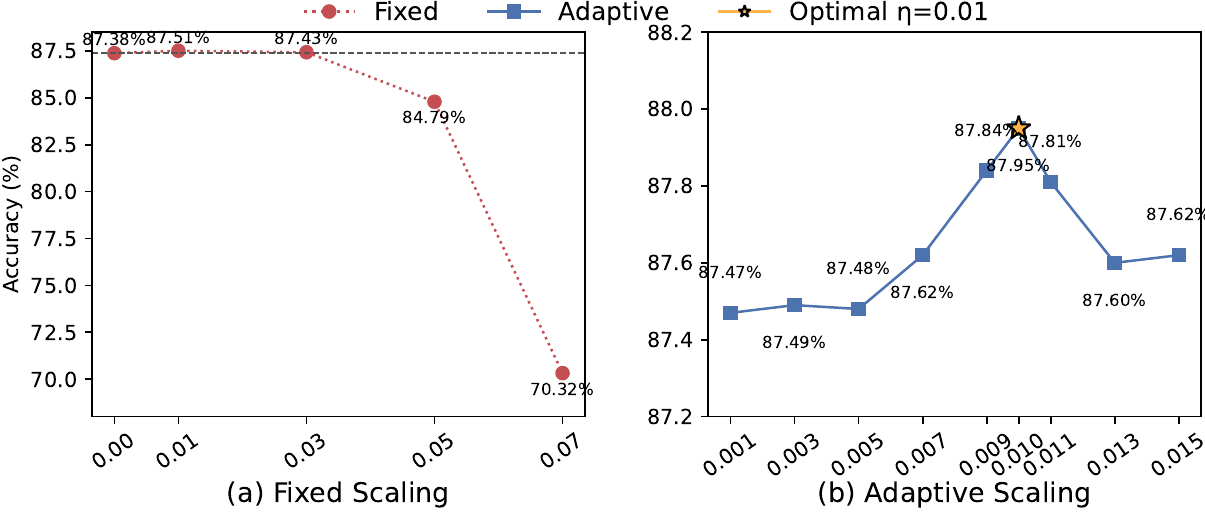}
    \caption{Ablation study on parameter scaling strategies. (a) is fixed scaling with constant factors across training iterations. (b) is adaptive scaling via parameter $\eta$ that dynamically adjusts the compensation strength (Eq.~\ref{eq:ags_rule}).}
    \label{fig:ablation_param}
\end{figure*}

\textbf{Ablation on Parameter $\eta$. }
As shown in Figure~\ref{fig:ablation_param}, the performance degradation of fixed scaling ($scale>0.05$, max -17.7\%) highlights the necessity of dynamic adaptation, while our adaptive scaling achieves peak accuracy (87.95\%) at $\eta=0.01$. The result confirms that our theory-driven design (Theorem~\ref{thm:optimal-lambda}) successfully balances gradient compensation and training stability.

\textbf{Ablation on Gradient Compensation Scope of DPGC. }
% We ablate gradient compensation scope on CIFAR-10 using ResNet20. As shown in Table~\ref{tab:ablation_right}, the baseline achieves 87.4\% accuracy. Compensating gradients of activation that are clipped by STE improves performance by +0.3\%, and +0.2\% by compensating unclipped ones, demonstrating that SURGE eliminates fixed-range gradient clipping constraints and enable better parameter updates. Compensating all gradients of activation improves performance by +0.6\%. 
We ablate the gradient compensation scope on CIFAR-10 using ResNet20. As detailed in Table~\ref{tab:ablation_right}, compensating \textit{only} gradients outside STE's clipping range ($|x|>1$) yields 87.7\% accuracy (0.3\% improvement over baseline), while compensating \textit{solely} within-range gradients ($|x|\leq1$) achieves 87.6\% (0.2\% improvement). This verifies that both clipped and preserved gradient components contribute to parameter optimization. When jointly compensating \textit{all} activation gradients through SURGE's adaptive integration, accuracy rises to 88.0\% (0.6\% improvement). This confirms that SURGE's design overcomes fixed-range clipping limitations in STE, enabling comprehensive gradient utilization.

\section{Conclusion}
This paper proposes a novel gradient compensation strategy that mitigates the STE-induced gradient mismatch through an auxiliary backpropagation. The proposed Dual-Path Gradient Compensator (DPGC) utilizes a dual-path architecture that ensures identical output to standard BNNs while providing less biased gradients for compensation. And the Adaptive Gradient Scaler (AGS) dynamically balances inter-branch gradient contributions via norm-based scaling. SURGE obtains the best performance over existing methods through main benchmarks in image classification, object detection, and language understanding tasks. 

% % Acknowledgements should only appear in the accepted version.
\section*{Acknowledgements}
The work was supported by the National Key Research and Development Program of China (No.2023YFC3306401) and Beijing Natural Science Foundation L244043. This research was also supported by National Natural Science Foundation of China 623B2016 and 62576018, and Zhejiang Provincial Natural Science Foundation of China under Grant No. LD24F020007.
% \textbf{Do not} include acknowledgements in the initial version of the paper

% If a paper is accepted, the final camera-ready version can (and usually should)
% include acknowledgements.  Such acknowledgements should be placed at the end of
% the section, in an unnumbered section that does not count towards the paper
% page limit. Typically, this will include thanks to reviewers who gave useful
% comments, to colleagues who contributed to the ideas, and to funding agencies
% and corporate sponsors that provided financial support.

% \clearpage
\section*{Impact Statement}
This paper proposes a training-time gradient compensation framework for Binary Neural Networks (BNNs) to improve optimization stability and accuracy under extreme quantization. The primary expected impact is enabling more efficient deployment of neural models on resource-constrained devices through reduced memory footprint and computation, which may lower energy consumption and operational cost for practical applications.

We do not anticipate direct safety or security risks introduced by the proposed method beyond those already associated with deploying machine learning models in real-world settings. The method does not introduce new data collection, does not require sensitive information, and does not change the functional scope of the underlying models; it mainly affects the training dynamics and can be removed at inference time. As with other model compression techniques, improved efficiency could facilitate wider deployment, and responsible use should follow standard best practices for dataset governance, evaluation, and monitoring in downstream applications.

\bibliography{example_paper}
\bibliographystyle{icml2026}

%%%%%%%%%%%%%%%%%%%%%%%%%%%%%%%%%%%%%%%%%%%%%%%%%%%%%%%%%%%%%%%%%%%%%%%%%%%%%%%
%%%%%%%%%%%%%%%%%%%%%%%%%%%%%%%%%%%%%%%%%%%%%%%%%%%%%%%%%%%%%%%%%%%%%%%%%%%%%%%
% APPENDIX
%%%%%%%%%%%%%%%%%%%%%%%%%%%%%%%%%%%%%%%%%%%%%%%%%%%%%%%%%%%%%%%%%%%%%%%%%%%%%%%
%%%%%%%%%%%%%%%%%%%%%%%%%%%%%%%%%%%%%%%%%%%%%%%%%%%%%%%%%%%%%%%%%%%%%%%%%%%%%%%
\clearpage
\renewcommand\thefigure{\Alph{figure}} 
\setcounter{figure}{0}
\renewcommand\thetable{\Alph{table}} 
\setcounter{table}{0}
\renewcommand{\thealgorithm}{\Alph{algorithm}}
\setcounter{equation}{0}
\renewcommand{\theequation}{\thesection.\arabic{equation}}
\appendix
\onecolumn
{\Large \textbf{APPENDIX}}
\section{Training Procedure}
The complete training procedure is summarized in Algorithm~\ref{alg:layer_training}. 
\label{appendix:training_procedure}
\begin{algorithm}[h]
% \small
\resizebox{0.95\linewidth}{!}{ 
\begin{minipage}{\linewidth}
\caption{Layer-wise Training with DPGC \& AGS}
\label{alg:layer_training}
\begin{algorithmic}[1]
\REQUIRE Layer input $x^{(l)}$, target $y$, learning rate $\alpha$, base scaling coefficient $\eta$, loss function $\mathcal{F}$
\ENSURE Trained binarized weights $\{W_b^{(l)}\}_{l=1}^L$
\STATE Initialize layer parameters $W_b^{(l)}, W_a^{(l)}$ \COMMENT{Binary \& auxiliary full-precision paths}
\STATE Initialize $\lambda^{(l,0)} \gets \frac{1}{\sqrt{|W_a^{(l)}|}}$ \COMMENT{Reciprocal sqrt of auxiliary weight cardinality $|W_a^{(l)}|$}

\FOR{iteration $t=1$ to $T$}
    \STATE \textbf{Forward Propagation (Layer $l$):}
    \STATE Compute binary path: $f_b^{(l)} \gets W_b^{(l)} \odot \mathrm{Sign}(x^{(l)})$
    \STATE Compute auxiliary path: $f_a^{(l)} \gets W_a^{(l)} \odot x^{(l)}$
    \STATE Generate compensator: $f_{ao}^{(l)} \gets \lambda^{(l,t-1)} \odot f_a^{(l)}$ \COMMENT{Previous scaling factor}
    \STATE Synthesis output: $\mathrm{out}^{(l)} \gets f_b^{(l)} - f_{ao}^{(l)}\downarrow + f_{ao}^{(l)}$ \COMMENT{Forward synthesis via gradient-decoupled decomposition, detach gradient at $\downarrow$}
    
    \STATE \textbf{Loss Computation:}
    \STATE Calculate Loss $\mathcal{L}$
    
    \STATE \textbf{Backward Propagation (Layer $l$):}
    \STATE Compute main branch gradients:
    \STATE $\quad g_b \gets \frac{\partial\mathcal{L}}{\partial f_b}\frac{\partial f_b}{\partial x}\Big|_{\text{STE}}$, $\quad g_{wb}^{(l)} \gets \frac{\partial\mathcal{L}}{\partial W_b^{(l)}}\Big|_{\mathrm{STE}}$ 
    \STATE Compute auxiliary branch gradients:
    \STATE $\quad g_a \gets \frac{\partial\mathcal{L}}{\partial f_{a}}\frac{\partial f_{a}}{\partial x}$, $\quad g_{wa}^{(l)} \gets \lambda^{(l,t-1)} \odot \frac{\partial\mathcal{L}}{\partial W_a^{(l)}}$ 
    \STATE Then we have: $\frac{\partial\mathcal{L}}{\partial x^{(l)}}=g_b+\lambda g_a$  
    
    \STATE \textbf{Adaptive Gradient Scaler:}
    \STATE Calculate norm ratio: $r \gets {\|g_b\|_2}/{(\|g_a\|_2 + \epsilon)}$
    \STATE Update scaling factor: $\lambda^{(l,t)} \gets \eta \cdot r$
    
    \STATE \textbf{Parameter Update (Layer $l$):}
    \STATE $W_b^{(l)} \gets W_b^{(l)} - \alpha \cdot g_{wb}^{(l)}$
    \STATE $W_a^{(l)} \gets W_a^{(l)} - \alpha \cdot g_{wa}^{(l)}$
\ENDFOR
\end{algorithmic}
\end{minipage}
}
\end{algorithm}

\section{Theoretical Foundations and Proofs}
\subsection{Assumptions Underlying the Moment Model and Theorem~\ref{thm:optimal-lambda}}
\label{assumptions}
In this subsection we make explicit the assumptions underlying
Definition~\ref{def:grad_stats} and Theorem~\ref{thm:optimal-lambda}.
Intuitively, since the derivative of \texttt{sign} is zero almost everywhere and
corresponds to a Dirac delta distribution at the origin in the sense of distributions,
it is natural to view $g^*$ as the gradient induced by an ``ideal'' surrogate that
captures this behavior, while practical rules such as STE provide tractable but biased
approximations.
\begin{assumption}[Ideal reference gradient from a surrogate family]
\label{ass:ideal-gradient}
Fix a binarization node whose pre-binarization activation is denoted by $x\in\mathbb{R}^d$.
Consider a family $\mathcal{S}$ of smooth surrogate functions $s:\mathbb{R}\to\mathbb{R}$
that approximate the non-differentiable \texttt{sign} function used in binarization.
For each $s\in\mathcal{S}$, let $\ell_s(\cdot;W)$ denote the population loss of the
corresponding surrogate network, and define the population risk as a function of this node input:
\[
\mathcal{L}_s(x;W) := \mathbb{E}_{\xi}\big[\,\ell_s(\xi;W)\,\big]\quad \text{with the backpropagated gradient taken w.r.t. }x.
\]
We assume that there exists a surrogate
$s^* \in \mathcal{S}$ that attains the smallest population loss within this family, and
define the associated reference (``better'') gradient at the current parameter $W$ as
\[
    g^* := \nabla_x \mathcal{L}_{s^*}(x;W).
\]
This $g^*$ is not observable in practice and we never require a closed-form expression
for it; it serves as an ideal target that practical surrogate gradients aim to approximate.
We assume that $g^*$ has finite second moments.
\end{assumption}
\begin{assumption}[Empirical gradients as random vectors]
\label{ass:random-gradients}
At a fixed parameter $W$, the empirical gradients $g_b, g_a \in \mathbb{R}^d$
obtained from a single mini-batch are modelled as random vectors whose randomness
comes from mini-batch sampling, data noise, and the stochastic optimization procedure.
All expectations $\mathbb{E}[\cdot]$ and variances $\mathrm{Var}(\cdot)$ in our analysis
are taken with respect to this randomness, and the empirical gradients have finite
second moments.
\end{assumption}

% \begin{assumption}[Directional Consistency under Magnitude Distortion]
% We acknowledge that the STE gradient $g_b$ suffers from severe magnitude distortion due to the clipping mechanism (where gradients outside $[-1, 1]$ are zeroed). However, we assume that $g_b$ remains statistically correlated with the descent direction of the ideal reference gradient $g^*$. Formally, the cosine similarity satisfies $\mathbb{E}[\cos(g_b, g^*)] \ge c > 0$ for some constant $c$. Under this view, the approximation error $\delta_b$ is dominated by the scale mismatch rather than purely random noise, necessitating the magnitude-aware compensation provided by our auxiliary branch.
% \end{assumption}

\begin{assumption}[Directional consistency and bounded relative bias]
\label{ass:dir_consistency}
We acknowledge that STE-based gradients may suffer magnitude distortion due to clipping.
We assume that the baseline surrogate gradient $g_b$ remains statistically correlated with the descent direction of the ideal reference gradient $g^*$, i.e.,
$\mathbb{E}[\cos(g_b, g^*)] \ge c > 0$ during the main training phase.
Moreover, we assume the \emph{relative bias ratio}
$\beta:=\|\delta_b\|_2/\|\mu_b\|_2$ is bounded and varies slowly after a short transient (layer-wise),
so it can be treated as approximately constant when deriving practical scaling rules.
\end{assumption}

\begin{assumption}[Isotropic, homoscedastic gradient noise]
\label{ass:isotropic-noise}
We decompose the empirical gradients as
\[
    g_b = \mathbb{E}[g_b] + \varepsilon_b, \qquad
    g_a = \mathbb{E}[g_a] + \varepsilon_a,
\]
where the noise terms satisfy $\mathbb{E}[\varepsilon_b] = \mathbb{E}[\varepsilon_a] = 0$.
We assume an isotropic, homoscedastic noise model:
there exist scalars $\sigma_b^2,\sigma_a^2 \ge 0$ such that
\[
    \mathrm{Var}(g_b) = \mathbb{E}[\varepsilon_b \varepsilon_b^\top] = \sigma_b^2 I_d,
    \qquad
    \mathrm{Var}(g_a) = \mathbb{E}[\varepsilon_a \varepsilon_a^\top] = \sigma_a^2 I_d,
\]
where $I_d$ is the $d$-dimensional identity matrix.
\end{assumption}

\subsection{Proof of Theorem~\ref{thm:optimal-lambda}}
\label{proof:optimal}
\begin{proof}
Expand the error expectation:
\begin{align}
\mathbb{E}\!\left[\|\tilde g-g^*\|_2^2\right]
&=\mathbb{E}\!\left[\|g_b+\lambda g_a-g^*\|_2^2\right] \nonumber\\
&=\|\mathbb{E}[g_b]+\lambda\mathbb{E}[g_a]-g^*\|_2^2
+\mathrm{tr}(\mathrm{Var}(g_b))+\lambda^2\mathrm{tr}(\mathrm{Var}(g_a)) \nonumber\\
&\quad +2\lambda\,\mathbb{E}\!\left[(g_b-\mathbb{E}[g_b])^\top (g_a-\mathbb{E}[g_a])\right]. 
\label{eq:bias_var_decomp}
\end{align}
By the dot-product uncorrelated assumption in Theorem~\ref{thm:optimal-lambda}, the last term in
\eqref{eq:bias_var_decomp} vanishes. 
Using $\mathbb{E}[g_b]=g^*-\delta_b$, we obtain
\begin{align}
\mathbb{E}\!\left[\|\tilde g-g^*\|_2^2\right]
&=\|\lambda\mathbb{E}[g_a]-\delta_b\|_2^2
+\mathrm{tr}(\mathrm{Var}(g_b))
+\lambda^2\mathrm{tr}(\mathrm{Var}(g_a)).
\end{align}
Differentiating w.r.t.\ $\lambda$ gives
\begin{align}
\nabla_\lambda\,\mathbb{E}\!\left[\|\tilde g-g^*\|_2^2\right]
= -2\langle\delta_b,\mu_a\rangle
+2\lambda\left(\|\mathbb{E}[g_a]\|_2^2+\mathrm{tr}(\mathrm{Var}(g_a))\right),
\end{align}
which yields the minimizer
\begin{align}
\label{eq:analytical_form}
\lambda^*=
\frac{\langle\delta_b,\mathbb{E}[g_a]\rangle}{\|\mathbb{E}[g_a]\|_2^2+\mathrm{tr}(\mathrm{Var}(g_a))}.
\end{align}
In particular, under $\mathrm{Var}(g_a)=\sigma_a^2 I_d$, we have $\mathrm{tr}(\mathrm{Var}(g_a))=d\sigma_a^2$.

\noindent\textbf{Practical approximation via dynamic analysis.}
Eq.~\eqref{eq:analytical_form} can be rewritten as
\begin{align}
\lambda^*
=\frac{\|\delta_b\|_2}{\|\mu_a\|_2}\cdot
\frac{\cos\theta}{1+\rho},
\qquad
\cos\theta:=\frac{\langle\delta_b,\mu_a\rangle}{\|\delta_b\|_2\|\mu_a\|_2},
\quad
\rho:=\frac{d\sigma_a^2}{\|\mu_a\|_2^2}.
\label{eq:lambda_factor}
\end{align}
To obtain a computable rule, we parameterize the unobserved bias magnitude via the
\emph{relative bias ratio} $\beta:=\|\delta_b\|_2/\|\mu_b\|_2$ (with $\mu_b=\mathbb{E}[g_b]$)
and treat $\beta\approx\kappa$ during the main training phase.
Plugging $\|\delta_b\|_2\approx \kappa\|\mu_b\|_2$ into \eqref{eq:lambda_factor} yields
\begin{align}
\lambda^* \approx \eta\,\frac{\|\mu_b\|_2}{\|\mu_a\|_2},
\qquad
\eta:=\frac{\kappa\,c_\theta}{1+\rho}.
\end{align}
Finally, replacing population quantities by mini-batch estimates gives
\begin{align}
\lambda_{\mathrm{AGS}}\approx \eta\,\frac{\|g_b\|_2}{\|g_a\|_2+\epsilon},
\end{align}
where $\epsilon>0$ is a numerical stabilizer.
Empirically, the inter-branch cosine similarity stays high over most of training and drops only near convergence
(Figure~\ref{fig:cosine_similarity}), while $\lambda$ quickly reaches a plateau in the toy study
(Figure~\ref{fig:5group_convergence}), supporting the use of a slowly-varying $\eta$ in practice.

\end{proof}

\begin{figure}[t]
    \centering
    \includegraphics[width=0.9\linewidth]{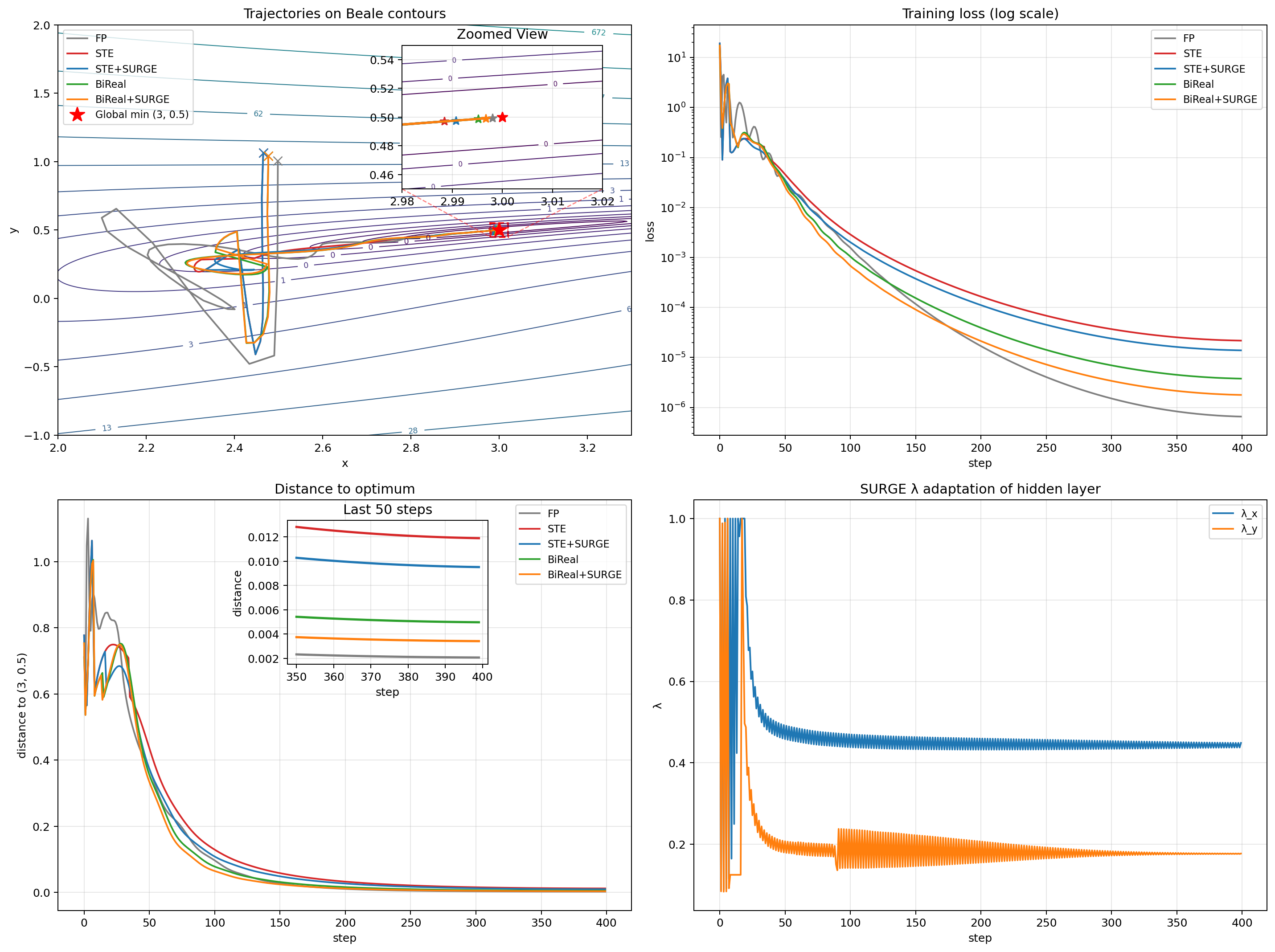}
    \caption{Comparison of five methods (FP, STE, STE+SURGE, Bi-Real, Bi-Real+SURGE) on Beale function: (a) trajectories, (b) loss, (c) distance to optimum, (d) SURGE's $\lambda$ adaptation of hidden layer. }
    \label{fig:5group_convergence}
\end{figure}

\begin{figure}[t]
    \centering
    \includegraphics[width=0.9\linewidth]{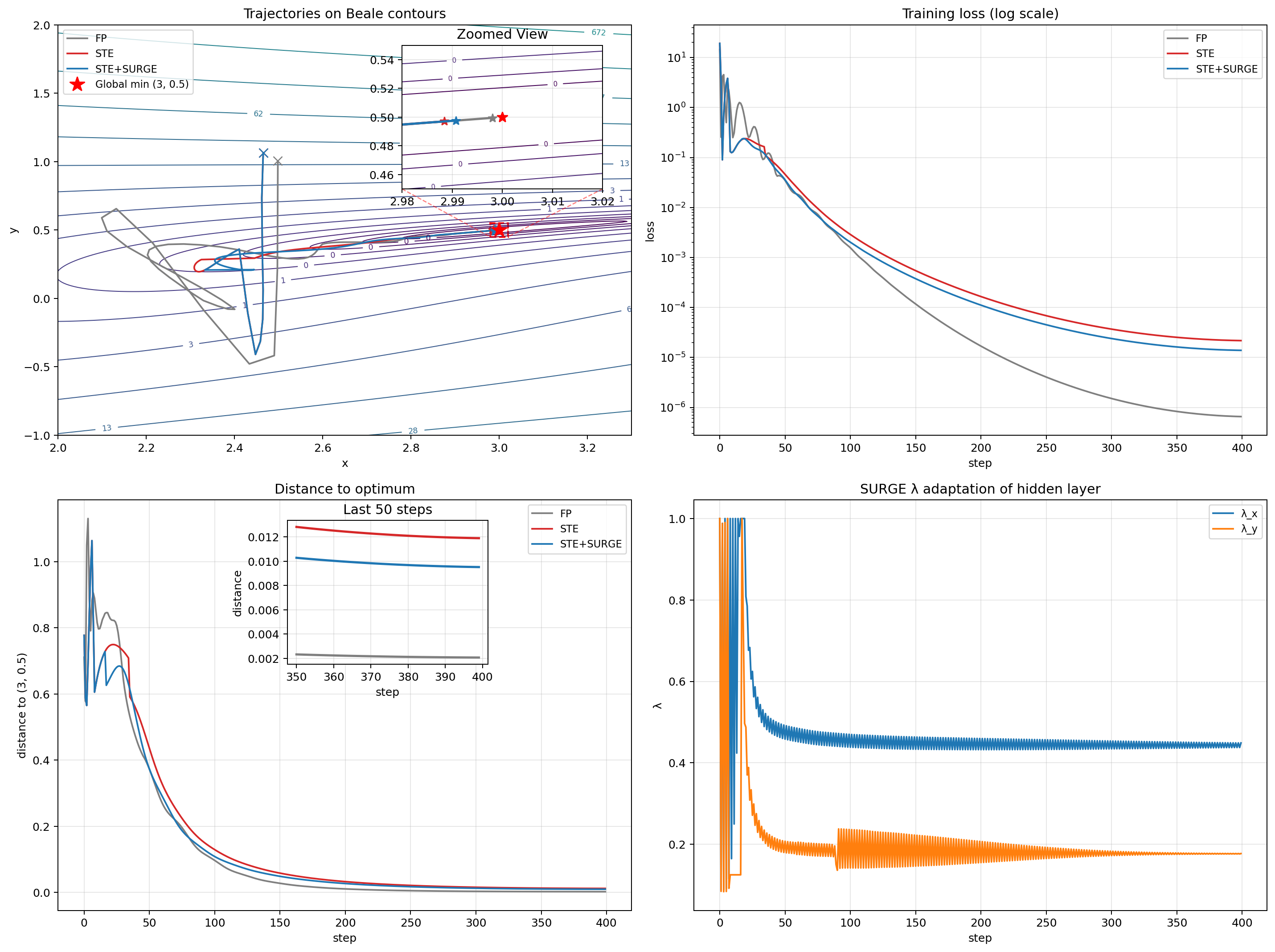}
    \caption{Comparison of three methods (FP, STE, STE+SURGE) on Beale function: (a) trajectories, (b) loss, (c) distance to optimum, (d) SURGE's $\lambda$ adaptation of hidden layer. }
    \label{fig:3group_convergence}
\end{figure}

\section{SURGE for Toy Problem and Visualizations}
\label{appendix:toy}
We implement a simple yet illustrative toy model to optimize the non-convex Beale function. The original model architecture consists of an input layer, a hidden layer with ReLU activation, and an output layer that produces 2D coordinates. In our experiments, we binarize the first and second linear layer.

\textbf{Convergence and loss curve of toy model.} As illustrated in Figure~\ref{fig:5group_convergence}, we compare five training methods under identical initialization: FP (Full-precision network), STE, STE+SURGE, Bi-Real, Bi-Real+SURGE. We also provide a focused three-group subset (FP, STE, STE+SURGE) shown in Figure~\ref{fig:3group_convergence} for clearer visualization. Specifically, for each binarized layer (based on STE / Bi-Real), SURGE adds a parallel full-precision layer and merges their outputs. We provide trajectory plot, loss curve, distance to optimum, adaptive scaling factor ($\lambda$). It can be seen that SURGE achieves better convergence performance, yielding lower loss compared to the control group without SURGE integration. And the scale factor $\lambda$ also reaches convergence.

\textbf{Parameter evolution of toy model.} As illustrated in Figure~\ref{fig:5group_param}, we provide a parameter evolution plot, where we track the Frobenius norms of the binary and auxiliary full-precision weights, as well as the learnable scaling factors $\alpha_w$, $\alpha_a$. We also provide a focused three-group subset (FP, STE, STE+SURGE) shown in Figure~\ref{fig:3group_param} for clearer visualization. As shown in the figure of Frobenius norm of weights, compared to the FP model, all binarized variants lie in a narrow band, but variants with SURGE maintain a slightly larger and more stable weight norm after the initial transient. As shown in the figure of learnable scaling factors, SURGE consistently pushes these scales slightly higher.

\textbf{Weight distribution.} As illustrated in Figure~\ref{fig:weight_distri}, we provide the weight distribution of one layer in ResNet-18 trained with CIFAR-10. It can be seen that weights around zero is less with SURGE than the counterpart without SURGE. The binarization results without SURGE is less robust to any robust disturbance \citep{xu2022recurrent_rbonn}, as sign(w) would more frequently flips.

\textbf{Cosine similarity.} As shown in Figure~\ref{fig:cosine_similarity}, we draw a figure of cosine similarity (on ResNet-20 trained with CIFAR-10) between the gradient of weights of main branch and auxiliary branch, averaged over layers and mini-batches. It can be seen that the cosine similarity is relatively high. During the main training phase, the similarity slowly decreases but remains in the range 0.8-0.9. Towards the very end of training, when the model has already entered a small local basin and gradients become very small, the cosine similarity drops more sharply. This is expected: the full-precision branch can still perform fine-grained adjustments inside the basin, whereas the binary branch is constrained by quantization, so the auxiliary branch likely compensates in different directions.

\textbf{Noise contrast experiment.} As shown in Figure~\ref{fig:4group_convergence}, we conducted experiments with added noise on the toy model to more intuitively demonstrate that our compensation is not merely noise. The results show that the convergence process with added noise becomes significantly more volatile, and the final convergence performance is worse.

\begin{remark}[Phase-wise alignment and interpretation]
\label{rem:phase_alignment}
The ``approximately stable'' alignment in Corollary~\ref{cor:norm-ratio} is intended to hold
over the main optimization stage after warm-up.
As observed in Figure~\ref{fig:cosine_similarity}, the cosine similarity between the two branches
remains high for most iterations, indicating that the auxiliary branch largely reinforces compatible
descent directions. Near convergence, the similarity decreases as the binarized branch becomes more
constrained within a quantization-induced basin, and the auxiliary (full-precision) gradient may
play a more prominent role in directional correction and fine-grained refinement.
We further observe a consistently positive cosine similarity between the gradients of inputs of the
two branches in Figure~\ref{fig:cosine_similarity_input}, suggesting that such compatibility also appears
in the backpropagated layer-wise gradient signals.
Empirically, the adaptive scale $\lambda$ quickly stabilizes after a brief transient (Figure~\ref{fig:5group_convergence}), supporting our "main-phase stability" approximation in Corollary~\ref{cor:norm-ratio}.
\end{remark}

\begin{figure}[t]
    \centering
    \includegraphics[width=0.9\linewidth]{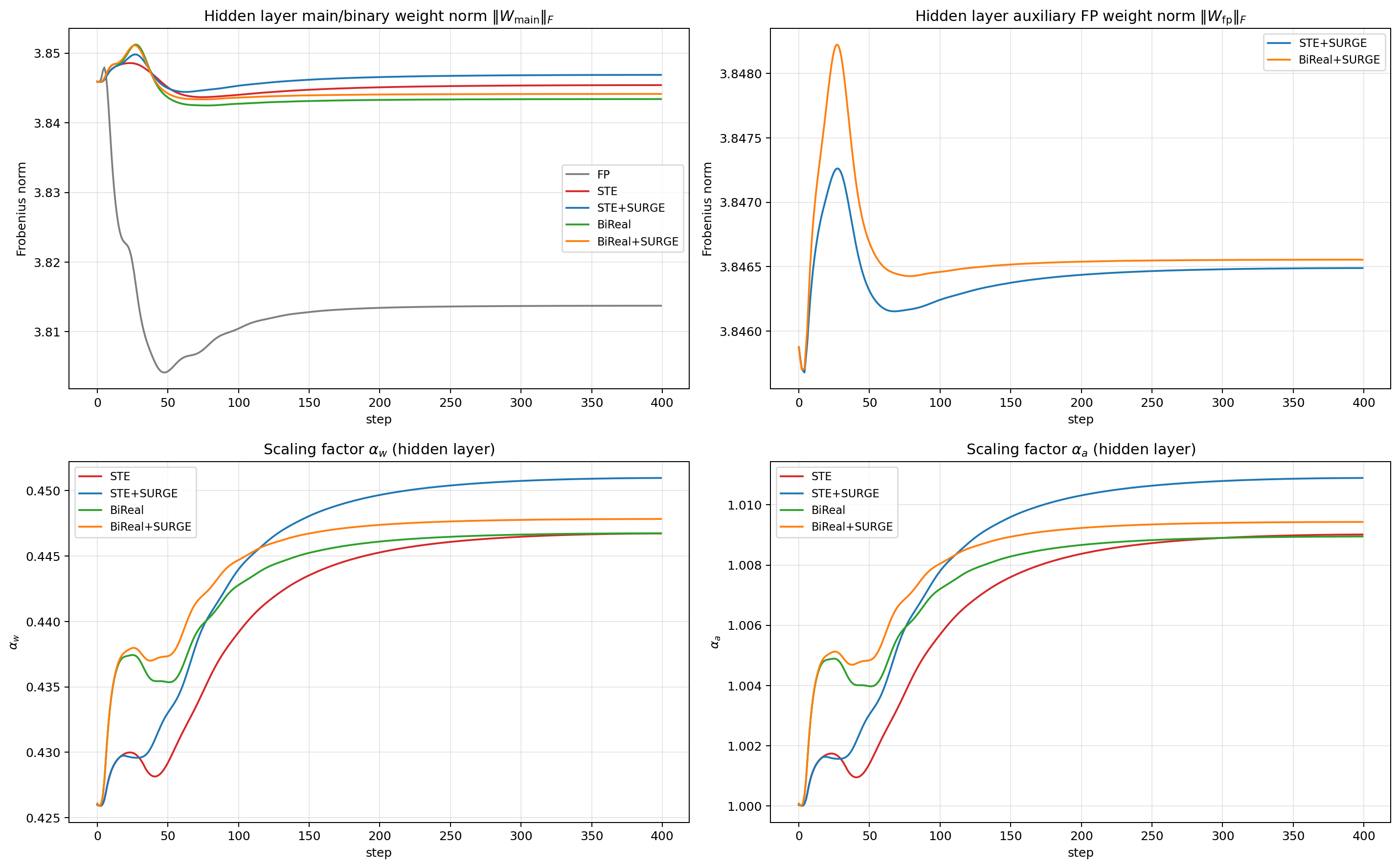}
    \caption{Comparison of five methods (FP, STE, STE+SURGE, Bi-Real, Bi-Real+SURGE) on Beale function: (a) weight norm of main branch, (b) weight norm of auxiliary branch, (c) scaling factor of weights, (d) scaling factor of activations. }
    \label{fig:5group_param}
\end{figure}

\begin{figure}[ht]
    \centering
    \includegraphics[width=0.9\linewidth]{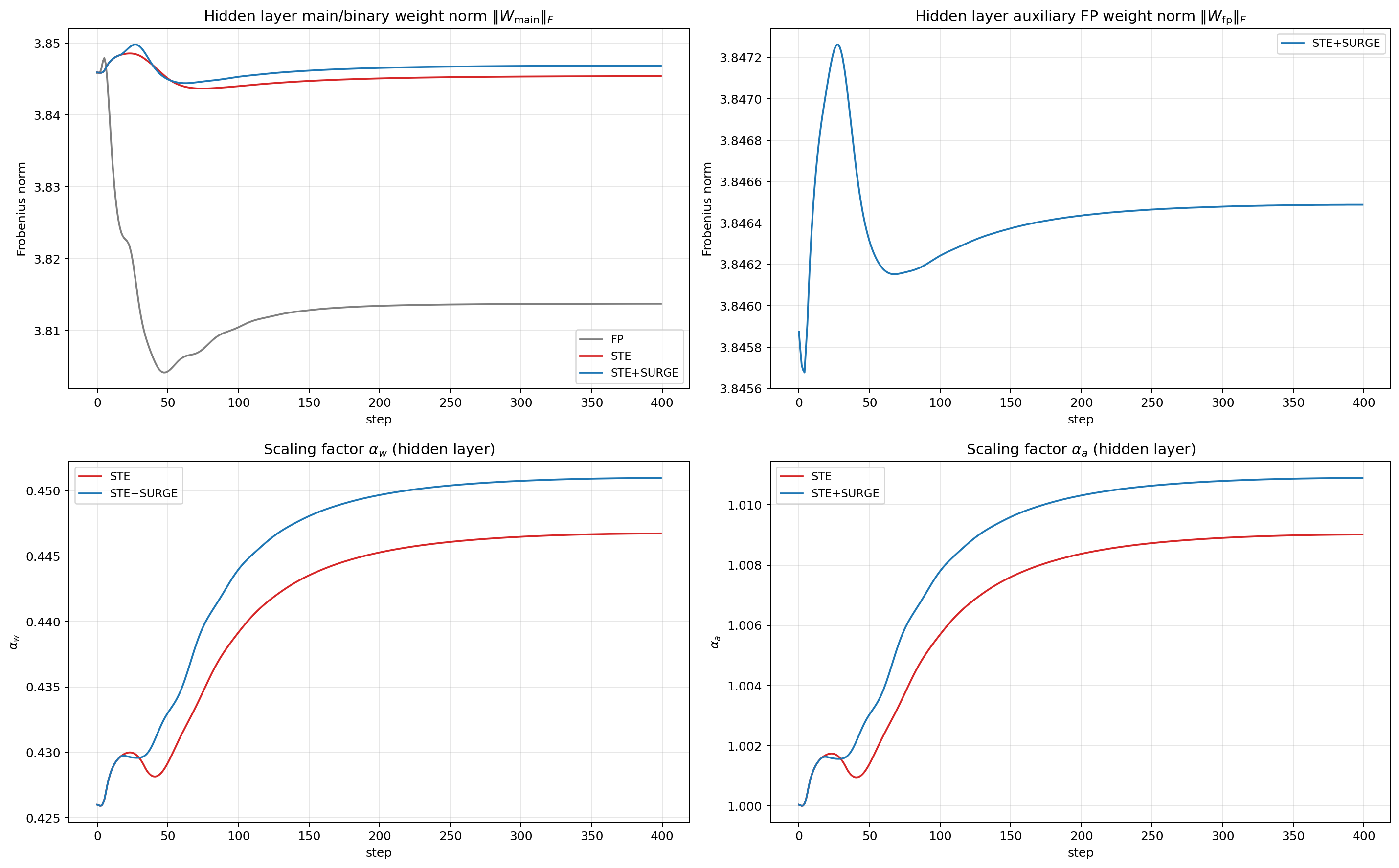}
    \caption{Comparison of three methods (FP, STE, STE+SURGE) on Beale function: (a) weight norm of main branch, (b) weight norm of auxiliary branch, (c) scaling factor of weights, (d) scaling factor of activations. }
    \label{fig:3group_param}
\end{figure}

\begin{figure}[ht]
    \centering
    \includegraphics[width=1.0\linewidth]{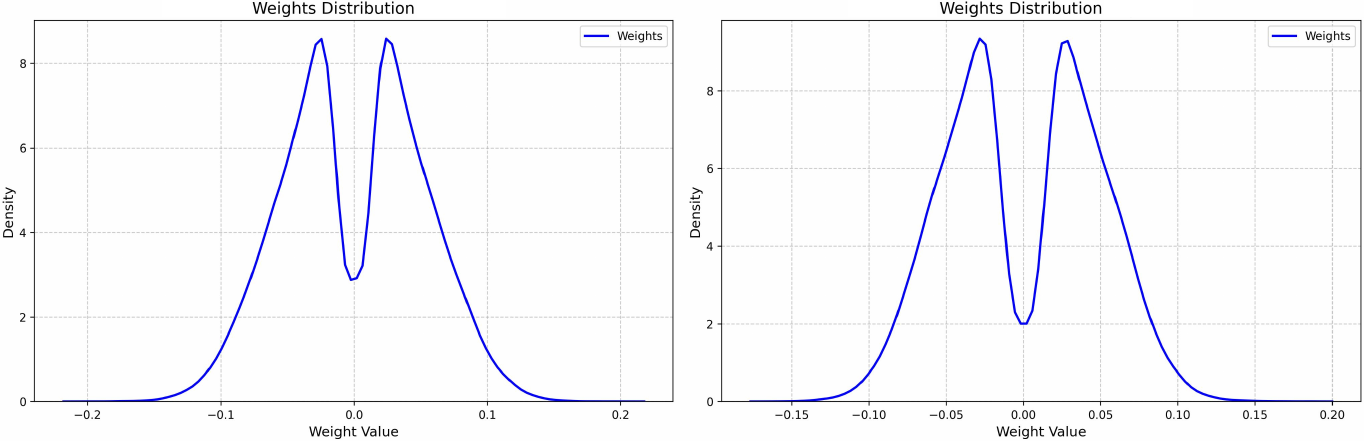}
    \caption{Weight distribution comparison of a ResNet-18 layer trained on CIFAR-10. (left) Baseline method; (right) SURGE (ours). }
    \label{fig:weight_distri}
\end{figure}

\begin{figure}[ht]
    \centering
    \includegraphics[width=.8\linewidth]{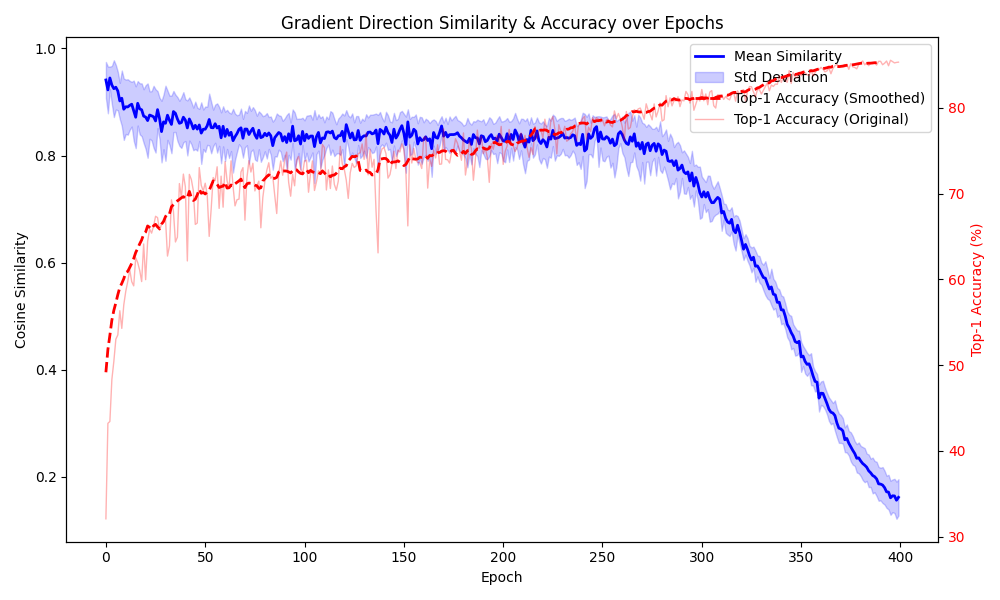}
    \caption{Cosine similarity (on ResNet-20 trained with CIFAR-10) between the gradient of weights of main branch and auxiliary branch, averaged over layers and mini-batches}
    \label{fig:cosine_similarity}
\end{figure}

\begin{figure}[ht]
    \centering
    \includegraphics[width=.9\linewidth]{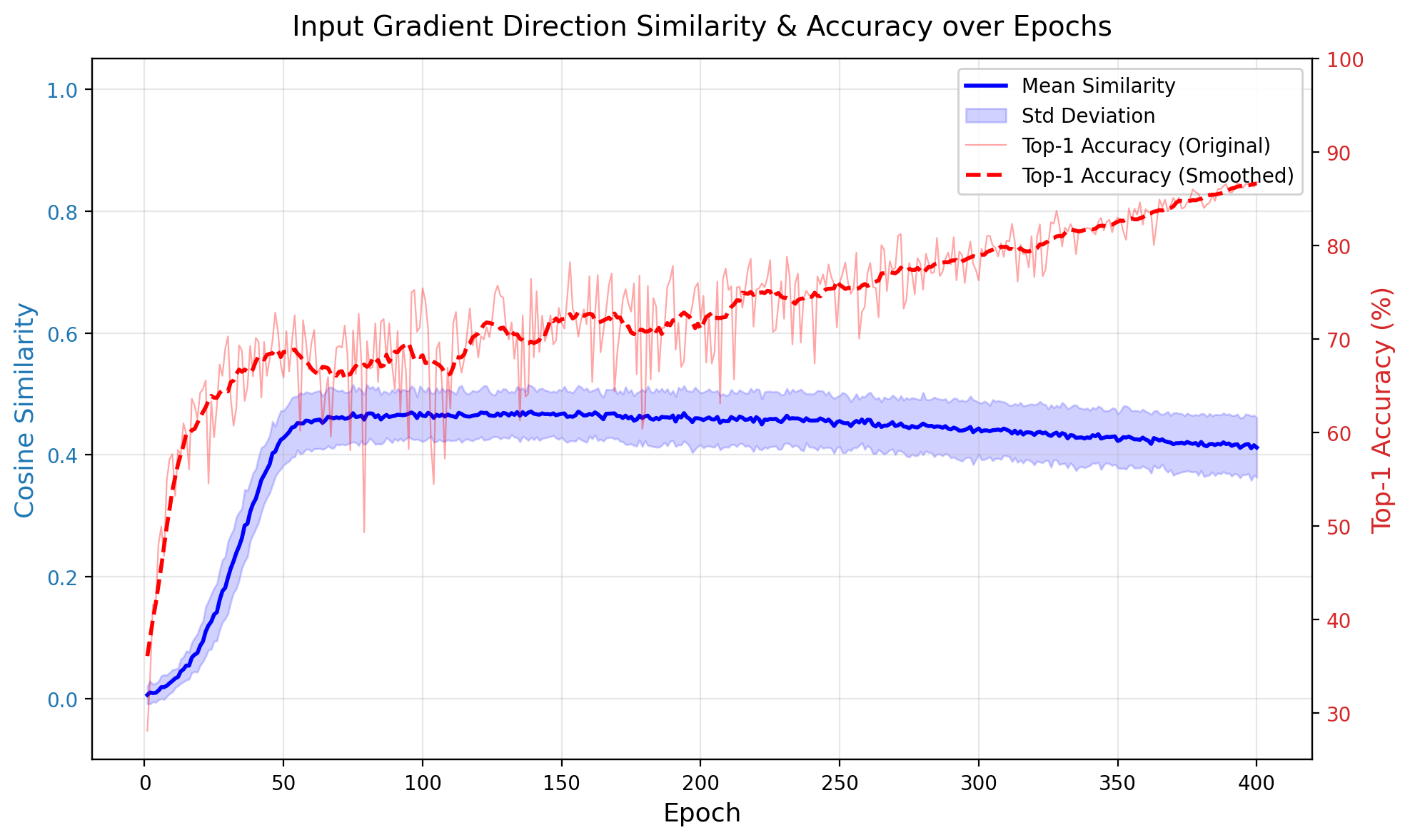}
    \caption{Cosine similarity (on ResNet-20 trained with CIFAR-10) between the gradient of inputs of main branch and auxiliary branch, averaged over layers and mini-batches}
    \label{fig:cosine_similarity_input}
\end{figure}

\begin{figure}[ht]
    \centering
    \includegraphics[width=0.9\linewidth]{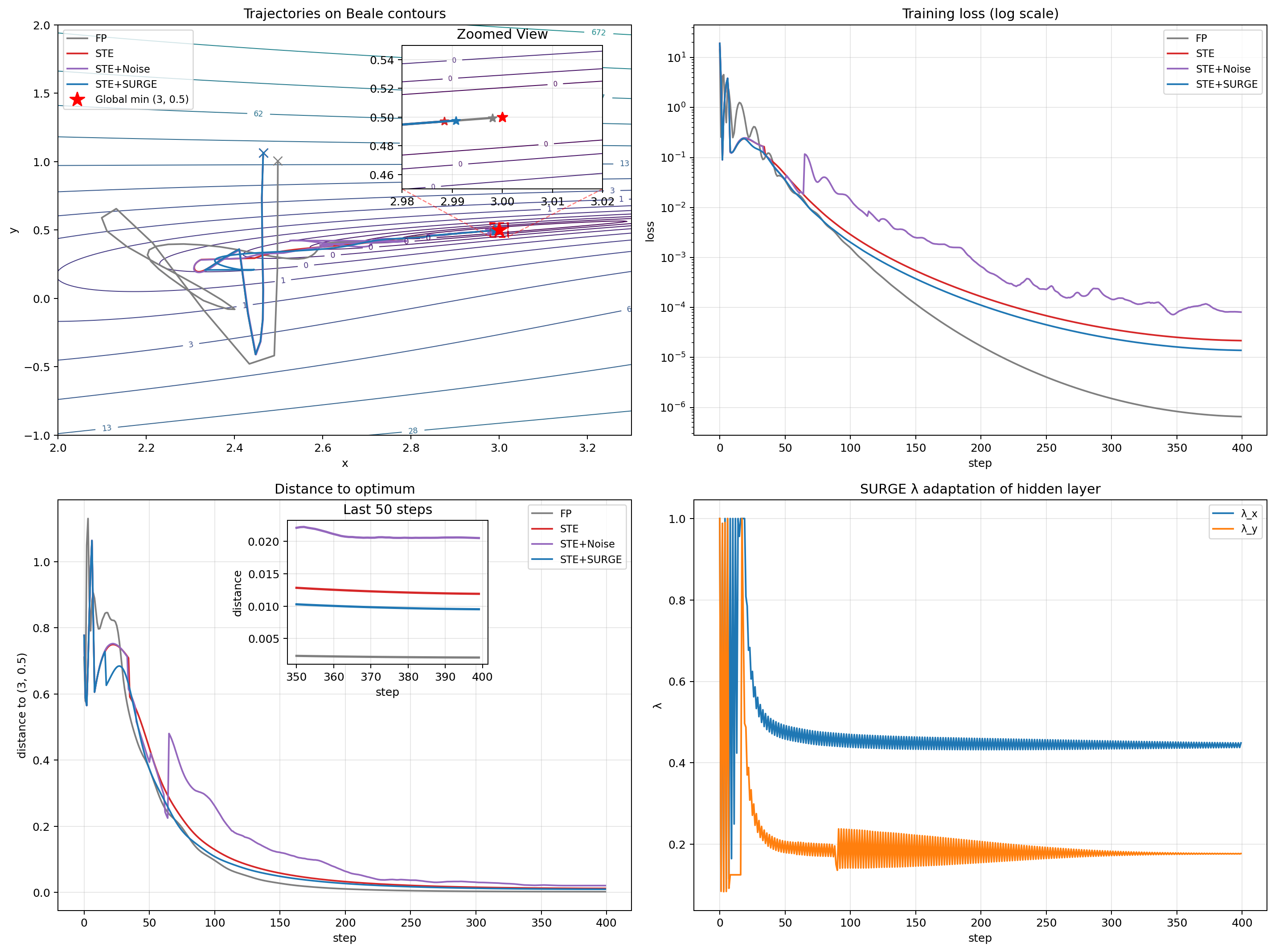}
    \caption{Comparison of five methods (FP, STE, STE+Noise, STE+SURGE) on Beale function: (a) trajectories, (b) loss, (c) distance to optimum, (d) SURGE's $\lambda$ adaptation of hidden layer. }
    \label{fig:4group_convergence}
\end{figure}

\section{Dataset, Data Augmentation, and Evaluating Metrics}
\label{appendix:dataset}
We evaluate on two standard image classification benchmarks, one object detection benchmark, and one suite of language understanding tasks to demonstrate the effectiveness: CIFAR-10 \citep{krizhevsky2009learning_cifar} (10k 32$\times$32 images with random cropping \& flipping), ImageNet-1K \citep{russakovsky2015imagenet} (1.28M training and 50k validation images at 224$\times$224 resolution via center crop), PASCAL VOC \citep{everingham2010pascalvoc} (around 16k training and 5k validation images across 20 classes with multi-scale resizing to 1500×900, 1000×600, and 666×400, random flipping at 0.5 ratio), and GLUE \citep{wang2018glue} (covers CoLA, SST-2, MRPC, STS-B, QQP, MNLI (m/mm), QNLI, and RTE, without data augmentation). 

We evaluate models using Top-1 and Top-5 accuracy for image classification, mean Average Precision at IoU=0.5 (mAP@0.5) for object detection, and task-specific GLUE metrics following the official protocol: Matthews correlation (MCC) for CoLA; accuracy for SST-2, MNLI (matched/mismatched), QNLI, and RTE; F1/Accuracy for MRPC and QQP; and Spearman correlations for STS-B. 

\section{Implementation Details}
\label{appendix:implementation_details}
\subsection{Model Details }
\textbf{CIFAR-10. }On CIFAR-10, we evaluate our method with ResNet-18/20 \citep{he2016deep_resnet} and VGG-Small \citep{simonyan2014very_vgg}. We binarize all convolutional and fully-connected layers except the first and last ones. 

\textbf{ImageNet-1K. }On ImageNet-1K, we binarize ResNet-18 and retain the first layer, shortcut, and last layer in the networks as real-valued following \citep{liu2018bireal_bnn}. We adopt the same model modification scheme as described in \citep{liu2020reactnet}. 

\textbf{PASCAL VOC. }On PASCAL VOC, we binarize Faster-RCNN with a ResNet-18 backbone. We keep the shortcut, first layer, and the last layer (the 1$\times$1 convolution layer of RPN and an FC layer of the bbox head) in the detectors as real-valued after implementing 1-bit CNNs. Following \citep{wang2020bidet}, we modify the network of ResNet-18 with an extra shortcut and PReLU \citep{he2015delving_prelu}. 

\textbf{GLUE. }On GLUE, we evaluate our method with BERT-base \citep{devlin2019bert}. We follow the previous work to binarize the word embedding layer, MHA and FFN in transformer layers, but leave full-precision classifier, position embedding layer, and token type embedding layer \citep{qinbibert, liu2022bit}. 

\subsection{Training Details }
\textbf{CIFAR-10. }On CIFAR-10, we train our models from scratch and following the setting in \citep{xu2021recu_bnn}, and the base scaling coefficient $\eta$ is set to $0.01$. 

\textbf{ImageNet-1K. }On ImageNet-1K, we follow two implementation setups for fair comparison. First, we employ \textbf{one-stage training} on ResNet-18 following the setting in \citep{xu2022recurrent_rbonn}, using Adam as the optimizer and a weight decay of $1e-5$. The initial learning rate is set to $5e-4$. The model is trained from scratch for 200 epochs with learning rates optimized by the annealing cosine learning rate schedule. Second, we employ \textbf{two-stage training} following the setting in \citep{liu2020reactnet}, using Adam as the optimizer. The network is supervised by a real-valued ResNet-34 teacher. In the first stage, the model is trained from scratch with binarized activation and real-valued convolution weights. We load the state dict from the first stage, and both the activation and weights are binarized in the second stage. The initial learning rate is set to $5e-4$, the same as one-stage training, and annealed to 0 by a linear descent scheduler. The base scaling coefficient $\eta$ is set to $0.001$. 

\textbf{PASCAL VOC. }On PASCAL VOC, we use ImageNet to pre-train the backbone of a 1-bit student, following \citep{liu2020reactnet}. The SGD optimizer is utilized, and the batch size is set as 4 for Faster-RCNN. We train the model in two stages. Only the backbone is binarized at the first stage. Then we binarize all layers in the second stage. Each stage counts 12 epochs. The learning rate is set as 0.004 and decays by multiplying 0.1 in the 9th and 11th epochs following \citep{xu2022ida_det}. The base scaling coefficient $\eta$ is set to $0.001$. 

\textbf{GLUE. }On GLUE, We follow \cite{liu2022bit} in adopting the experimental setting of \cite{devlin2019bert}. We use the Adam as our optimizer, and we take more training epochs for every quantization method on each tasks to have a sufficient training, which is 50 for CoLA, 20 for MRPC, STS-B and RTE, 10 for SST-2 and QNLI, 5 for MNLI and QQP. We distill binary models using full-precision teacher without using multi-distill technique. 

Furthermore, \textbf{\emph{we have provided our code}} in the supplementary materials, which contains the full implementation of our method and training scripts to facilitate easy replication and future research.

\section{Additional GLUE Analyses}
\label{appendix:additional_glue}

\subsection{Training under Fixed Wall-Clock Budgets}
\label{appendix:fixed_budget_glue}

To further examine whether the improvement of SURGE comes from more effective optimization
rather than simply longer training, we additionally evaluate BiT and BiT+SURGE under the same
fixed wall-clock budget. Specifically, for each GLUE task, both methods are trained on identical
hardware with the same task-specific training duration, and we compare the performance reached within that budget.

As shown in Table~\ref{tab:glue_fixed_budget}, BiT+SURGE still outperforms the BiT baseline
under the same time budget, improving the average GLUE score from 65.2 to 66.7. This indicates
that, despite its higher per-step cost, SURGE provides more effective optimization updates within
a controlled training-time budget.

\begin{table}[ht]
\centering
\caption{Results under fixed wall-clock budgets on GLUE. For each task, both methods are trained
on the same hardware with the same task-specific training duration. 
% The average score is computed over GLUE tasks by first averaging MNLI-m and MNLI-mm as one MNLI score.
}
\label{tab:glue_fixed_budget}
\resizebox{\linewidth}{!}{
\begin{tabular}{lcccccccccc}
\toprule
\textbf{Method} &
\makecell{\textbf{MNLI-m}\\\textbf{(10800s)}} &
\makecell{\textbf{MNLI-mm}\\\textbf{(10800s)}} &
\makecell{\textbf{QQP}\\\textbf{(7200s)}} &
\makecell{\textbf{QNLI}\\\textbf{(10800s)}} &
\makecell{\textbf{SST-2}\\\textbf{(7200s)}} &
\makecell{\textbf{CoLA}\\\textbf{(3600s)}} &
\makecell{\textbf{STS-B}\\\textbf{(1800s)}} &
\makecell{\textbf{MRPC}\\\textbf{(1500s)}} &
\makecell{\textbf{RTE}\\\textbf{(900s)}} &
\textbf{Avg.} \\
\midrule
BiT & 72.0 & 72.7 & 82.7 & 81.5 & 85.6 & 23.0 & 45.0 & 74.8 & 57.0 & 65.2 \\
BiT+SURGE & 72.1 & 72.3 & 83.3 & 82.4 & 86.5 & 21.4 & 54.3 & 78.7 & 54.5 & \textbf{66.7} \\
\bottomrule
\end{tabular}
}
\end{table}

\subsection{Effect of Explicit Auxiliary-Branch Alignment}
\label{appendix:auxiliary_alignment_glue}

SURGE uses an auxiliary full-precision branch to compensate for the truncated first-order gradient caused by binarization, while the detach trick prevents the auxiliary branch from changing the forward output value. Therefore, the auxiliary branch is designed to provide complementary gradient information, rather than to mimic the binary branch in the forward computation.

We further test a variant that explicitly encourages the auxiliary branch output to align with the binary branch output. As shown in Table~\ref{tab:glue_auxiliary_alignment}, this explicit alignment does not further improve the performance over BiT+SURGE. This suggests that forcing the auxiliary branch to behave too similarly to the binary branch may reduce the diversity of the compensation signal, thereby weakening its ability to correct the STE-induced gradient bias.

\begin{table}[ht]
\centering
\caption{GLUE dev-set results for BiT, BiT+SURGE, and BiT+SURGE+Align. All values are reported
in percentage form. 
% The average score is computed over GLUE tasks by first averaging MNLI-m and MNLI-mm as one MNLI score.
}
\label{tab:glue_auxiliary_alignment}
\resizebox{\linewidth}{!}{
\begin{tabular}{lcccccccccc}
\toprule
\textbf{Method} & \textbf{MNLI-m} & \textbf{MNLI-mm} & \textbf{QQP} & \textbf{QNLI} &
\textbf{SST-2} & \textbf{CoLA} & \textbf{STS-B} & \textbf{MRPC} & \textbf{RTE} & \textbf{Avg.} \\
\midrule
BiT & 77.0 & 77.5 & 85.4 & 85.5 & 87.8 & 23.6 & 68.0 & 79.4 & 58.1 & 70.6 \\
BiT+SURGE & 77.3 & 77.5 & 87.1 & 86.2 & 88.6 & 24.1 & 71.7 & 80.6 & 60.6 & \textbf{72.0} \\
BiT+SURGE+Align & 77.1 & 77.3 & 85.9 & 85.9 & 87.8 & 23.6 & 71.3 & 79.2 & 58.8 & 71.2 \\
\bottomrule
\end{tabular}
}
\end{table}

\section{Overhead and Deployment Efficiency}
\label{overhead}
SURGE introduces modest additional overhead during training while eliminating any extra inference cost, since the auxiliary branch is discarded after training.

\subsection{Training overhead. }
\textbf{CNNs. }We conduct a comparison of training time (10 epochs) and memory (batchsize 256/GPU) among SURGE, Bi-Real Net, ReActNet, and RBONN under one-stage training on ImageNet. As shown in Table~\ref{tab:overhead_one_stage}, our results demonstrate that while the full-precision branch introduces modest overhead (+25\% training time, +7.6\% memory vs RBONN), SURGE can deliver significant accuracy improvements against other SOTA methods (+0.63\% accuracy vs RBONN).

\begin{table}[ht]
\centering
\caption{Overhead comparison of training time (10 epochs) and memory (batchsize 256/GPU). Accuracy is reported after 10 epochs of training. * denotes a simple cost-reducing variant. }
\label{tab:overhead_one_stage}
\begin{tabular}{lccc}
\toprule
\textbf{Method} & \textbf{Training Time (min)} & \textbf{GPU Memory (2 GPUs)} & \textbf{Accuracy (\%)} \\
\midrule
Bi-Real Net & 143 & 19153 MiB$\times$2 & 38.73 \\
ReActNet & 156 & 20923 MiB$\times$2 & 45.86 \\
RBONN & 160 & 21005 MiB$\times$2 & 46.65 \\
\textbf{SURGE} & 200 & 22597 MiB$\times$2 & \textbf{47.28} \\
\textbf{SURGE*} & 177 & 22295 MiB$\times$2 & 47.21 \\
\bottomrule
\end{tabular}
\end{table}

\textbf{Transformers. }We conduct a comparison of training time (1 epoch) and memory consumption between SURGE and baseline (BiT) for BERT quantization on each task of GLUE. Following BiT, we employ task-specific batch sizes during training to optimize performance across different tasks. As shown in Table~\ref{tab:overhead_bert}, SURGE introduces acceptable additional training overhead (+17\% avg time, +22\% avg memory) while delivering significant accuracy improvements (+1.38\% avg accuracy). Notably, when the baseline employs larger batch sizes (32) with substantial memory consumption (12429 MB) on QQP dataset, SURGE adds only minimal additional overhead (+10\% memory).

\begin{table}[ht]
\centering
% \small
\caption{Comparison of Time, Memory, and Final Accuracy between Baseline and SURGE of training binarized BERT across GLUE tasks}
\label{tab:overhead_bert}
\begin{tabular}{l|l|c|c|c|c|c|c|c|c}
\hline
& \textbf{Method} & \textbf{CoLA} & \textbf{\makecell{MNLI\\(m/mm)}} & \textbf{MRPC} & \textbf{QNLI} & \textbf{QQP} & \textbf{RTE} & \textbf{SST-2} & \textbf{STS-B} \\
\hline
Batch size & --------- & 16 & 16 & 8 & 8 & 32 & 8 & 8 & 8 \\
\hline
\multirow{2}{*}{Time (1 epoch)/s} & Baseline & 107 & 5025 & 92 & 2590 & 3263 & 61 & 1583 & 144 \\
& SURGE & 122 & 6220 & 109 & 3050 & 4110 & 72 & 1755 & 158 \\
\hline
\multirow{2}{*}{Memory (MB)} & Baseline & 5701 & 8175 & 6051 & 6051 & 12429 & 6051 & 4617 & 6049 \\
& SURGE & 7227 & 9649 & 7475 & 7483 & 13731 & 7475 & 5957 & 7473 \\
\hline
\multirow{2}{*}{Final accuracy} & Baseline & 23.56 & 77.05/77.46 & 79.41 & 85.48 & 85.40 & 58.12 & 87.84 & 67.97 \\
& SURGE & 24.11 & 77.27/77.53 & 80.64 & 86.23 & 87.12 & 60.65 & 88.65 & 71.70  \\
\hline
\end{tabular}
\end{table}

Current BNN deployments predominantly target edge devices where inference efficiency is significant. Consequently, state-of-the-art methods in this domain prioritize two key metrics: (1) achievable accuracy under extreme quantization constraints, and (2) real-world inference latency on resource-limited hardware. Training efficiency remains secondary in established BNN research paradigms. Our method introduces modest additional overhead during training and does not impact deployment.

\subsection{Deployment efficiency. }
SURGE discards all auxiliary branches after training, maintaining identical resource requirements to standard binary networks while delivering stable accuracy gains. Based on our prior experience, we implement the 1-bit models on ODROID C4, which has a 2.016 GHz 64-bit quad-core ARM Cortex-A55. By evaluating its real speed in real-world mobile device, the deployment efficiency of SURGE is proven. We leverage the SIMD instruction SSHL on ARM NEON to make the inference library BOLT \citep{feng2021bolt} compatible with SURGE. We compare SURGE to the real-valued backbone in Table~\ref{tab:efficiency_comparison}. We can see that SURGE's inference speed is substantially faster with the highly efficient BOLT library. For example, the acceleration rate achieves about 4.1× on ResNet18.

\begin{table}[ht]
\centering
\caption{Deployment efficiency. }
% \small
% \setlength{\tabcolsep}{2.5pt}
\label{tab:efficiency_comparison}
\begin{tabular}{lcccccc}
\toprule
\textbf{Backbone} & \textbf{Method} & \textbf{\#bit (W/A)} & \textbf{Size (MB)} & \textbf{Memory Saving} & \textbf{Latency (ms)} & \textbf{Acceleration Rate} \\
\midrule
\multirow{2}{*}{ResNet-18} 
& Real-valued & 32/32 & 42.7 & -- & 276.8 & -- \\
& \textbf{SURGE} & \textbf{1/1} & \textbf{1.7} & \textbf{25.1$\times$} & \textbf{67.8} & \textbf{4.1$\times$} \\
\bottomrule
\end{tabular}
\end{table}

\section{Limitations and Future Works}
\label{limitations}
While SURGE improves gradient estimation through its dual-path design, this approach inherently requires the retention of auxiliary full-precision parameters $W_a^{(l)}$ throughout the training phase to enable gradient compensation. This architectural choice introduces two practical considerations: (1) a temporary increase of parameter memory footprint during backward propagation compared to conventional BNN implementations, and (2) additional computational overhead from parallel path gradient calculations during optimization. In Appendix~\ref{overhead}, we have conduct experiments to quantify the training overhead of binarizing CNNs and Transformers.  Notably, these costs are strictly confined to the training phase. During inference, the auxiliary branches are discarded, restoring the original binary architecture's computational efficiency and memory footprint without any residual overhead.

For architecture, our framework enables exploration of auxiliary structure designs (\eg, via efficient structure design, low-rank decomposition, or saliency-aware layer compensation) to reduce computational overhead. Here we propose a potential method to decrease training overhead. By simply replacing auxiliary convolutions from 3×3 to 1×1 kernels (SURGE* variant in Table~\ref{tab:overhead_one_stage}), we achieve 14.4\% training time reduction (177min vs 200min over 160min), 1.5\% memory savings (22295MB vs 22597MB over 21005MB), while still bringing accuracy improvements against other SOTA methods (+0.56\% accuracy vs RBONN). This validates that our framework inherently supports efficient re-engineering, and such architectural explorations constitute promising future directions.

% \section{Use of Large Language Models}
% We used large language models (LLMs) only to aid and polish writing (\eg, grammar, wording, and minor \LaTeX{} fixes). No algorithms, analyses, results, or claims were generated by LLMs; all technical content and decisions are by the authors. LLMs were not used to create/label data, and no evaluation items were exposed. Draft snippets were provided as prompts, and all outputs were manually reviewed before inclusion. 

%%%%%%%%%%%%%%%%%%%%%%%%%%%%%%%%%%%%%%%%%%%%%%%%%%%%%%%%%%%%%%%%%%%%%%%%%%%%%%%
%%%%%%%%%%%%%%%%%%%%%%%%%%%%%%%%%%%%%%%%%%%%%%%%%%%%%%%%%%%%%%%%%%%%%%%%%%%%%%%

\end{document}